\documentclass[english]{article}

\makeatletter

\usepackage[margin =1in]{geometry}

\usepackage{amsmath,amsthm,amssymb,amsfonts,mathtools}
\usepackage{bbm} %
\usepackage{blkarray} %
\usepackage{stackengine} %
\usepackage{enumitem}
\usepackage[square,sort,comma,numbers]{natbib}
\usepackage{dsfont} %
\usepackage[normalem]{ulem} %
\usepackage{appendix}

\numberwithin{equation}{section}

\makeatletter
\newcommand{\longdash}[1][2em]{%
  \makebox[#1]{$\m@th\smash-\mkern-7mu\cleaders\hbox{$\mkern-2mu\smash-\mkern-2mu$}\hfill\mkern-7mu\smash-$}}
\makeatother
\newcommand{\omitskip}{\kern-\arraycolsep}

\newcommand{\E}{{\mathbb E}}

\newcommand{\R}{{\mathbb R}}

\newcommand{\cA}{{\mathcal A}}
\newcommand{\cB}{\mathcal{B}}

\newcommand{\cN}{\mathcal{N}}
\newcommand{\cP}{\mathcal{P}}

\newcommand{\cS}{\mathcal{S}}

\newcommand{\cX}{\mathcal{X}}

\newcommand{\dd}{\textup{d}}

\newcommand{\Amax}{A_{\max}}
\newcommand{\bmax}{b_{\max}}
\newcommand{\gammin}{\gamma_{\min}}
\newcommand{\gammax}{\gamma_{\max}}
\newcommand{\law}{\mathcal{L}}
\newcommand{\bigO}{\mathcal{O}}

\newcommand{\var}{\operatorname{Var}}

\newcommand{\tr}{\operatorname{Tr}}

\usepackage[unicode=true,pdfusetitle, bookmarks=true,bookmarksnumbered=true,bookmarksopen=false, breaklinks=true,pdfborder={0 0 0},pdfborderstyle={},backref=page,colorlinks=true,citecolor=blue]{hyperref}
\usepackage{comment} %
\usepackage{cancel}

\usepackage{graphicx}
\usepackage{wrapfig}
\usepackage{tikz} %
\usepackage{float}
\usepackage{subcaption}
\usepackage{caption}
\usepackage{multirow}

\usepackage{algorithm2e}
\RestyleAlgo{ruled}

\newtheorem{thm}{Theorem}[section]
\newtheorem{definition}[thm]{Definition}
\newtheorem{proposition}[thm]{Proposition}

\newtheorem{cor}[thm]{Corollary}

\newtheorem{assumption}{Assumption}

\newtheorem*{question*}{Question}
\newtheorem{claim}{Claim}
\newtheorem*{claim*}{Claim}

\global\long\def\BarA{\Bar A}%
\global\long\def\mS{\mathcal{S}}%

\global\long\def\Upsilon{\upsilon}%

\makeatother

\begin{document}

\title{Effectiveness of Constant Stepsize in Markovian LSA and Statistical Inference}

\author{Dongyan (Lucy) Huo,\texorpdfstring{$^\mathsection$}{} Yudong Chen,\texorpdfstring{$^\dagger$}{} Qiaomin Xie,\texorpdfstring{$^\ddagger$}{}%
    \texorpdfstring{\footnote{Emails: \texttt{dh622@cornell.edu}, \texttt{yudong.chen@wisc.edu}, \texttt{qiaomin.xie@wisc.edu}}\\~\\
	\normalsize $^\mathsection$School of Operations Research and Information Engineering, Cornell University\\
	\normalsize $^\dagger$Department of Computer Sciences, University of Wisconsin-Madison\\
	\normalsize $^\ddagger$Department of Industrial and Systems Engineering, University of Wisconsin-Madison}{}
}

\date{}

\maketitle

\begin{abstract}
In this paper, we study the effectiveness of using a constant stepsize in statistical inference via linear stochastic approximation (LSA) algorithms with Markovian data. After establishing a Central Limit Theorem (CLT), we outline an inference procedure that uses averaged LSA iterates to construct confidence intervals (CIs). Our procedure leverages the fast mixing property of constant-stepsize LSA for better covariance estimation and employs Richardson-Romberg (RR) extrapolation to reduce the bias induced by constant stepsize and Markovian data. 
We develop theoretical results for guiding stepsize selection in RR extrapolation, and identify several important settings where the bias provably vanishes even without extrapolation. 
We conduct extensive numerical experiments and compare against classical inference approaches. Our results show that using a constant stepsize enjoys easy hyperparameter tuning, fast convergence, and consistently better CI coverage, especially when data is limited.
\end{abstract}

\section{Introduction}
Stochastic approximation (SA) algorithms use stochastic updates to iteratively approximate the solution to fixed-point equations. SA has wide applications, such as the stochastic gradient descent (SGD) algorithm for loss minimization and the Temporal Difference (TD) learning algorithm in reinforcement learning (RL) \citep{Sutton1988-td}.
Classical works on SA typically assume the stepsize sequence $\alpha_k$ is square-summable and diminishing, i.e., $\sum_t\alpha_t=\infty$ and $\sum_t\alpha_t^2<\infty$, under which asymptotic almost-sure convergence is well studied \citep{Robbins51-Monro-SA, Blum54-SA, borkar2000-ode-sa}. 
Constant stepsize has gained popularity recently, particularly among practitioners, due to its fast initial convergence and easy hyperparameter tuning.
A growing line of works studies the convergence properties of SA under constant stepsize, establishing upper bounds on the mean-squared error (MSE) \citep{Lakshminarayanan18-LSA-Constant-iid, srikant-ying19-finite-LSA, Mou20-LSA-iid, Mou21-optimal-linearSA} as well as weak convergence results \citep{Dieuleveut20-bach-SGD, Yu21-stan-SGD, huo2022bias}.

Recent works have explored using SA and SGD iterates to perform statistical inference, e.g., constructing confidence intervals (CIs) around a point estimate \citep{li2017-constantine, ChenXi2020, Li2022-Inference-SGD, LeeLiaoSeoShin2022-SGD-RS, liu2023-sgd-inference}. This approach is computationally cheap and scales well with the size and dimension of the dataset: SA updates are computed iteratively, without storing or multiple passes over the dataset. In comparison, classical inference techniques, such as bootstrapping, often require storing the entire dataset and repeating the estimation procedure, which may have prohibitive computational costs.

The aforementioned works on using SA for inference have focused on the diminishing stepsize paradigm, for which a mature convergence theory exists, ensuring the asymptotic correctness of the inference results. In contrast, constant-stepsize SA lacks last-iterate almost sure convergence: recent works have shown that the iterates converge only in distribution; moreover, the limit distribution may have a nonzero asymptotic bias due to the nonlinearity of the SA updates~\citep{Dieuleveut20-bach-SGD} or the underlying Markovian data \citep{huo2022bias}, and this bias cannot be eliminated by iterate averaging. Partly due to these considerations, inference with constant stepsize SA iterates has been largely overlooked in the literature. 

We study the effectiveness of statistical inference using constant-stepsize SA iterates. 
We focus on linear stochastic approximation (LSA), i.e., $\theta_{t+1}=\theta_t+\alpha(A(x_t)\theta_t+b(x_t))$, with a constant stepsize $\alpha$ and Markovian data $(x_t)_{t\geq0}$.
Our main contributions are as follows.
We first establish a Central Limit Theorem (CLT) for averaged Markovian LSA iterates. Built upon the CLT, we outline an inference procedure using averaged iterates and batch-mean covariance estimates. Our procedure leverages the fast mixing property of constant-stepsize updates for better covariance estimation and employs Richardson-Romberg (RR) extrapolation for bias reduction. 
We study two stepsize schemes in RR extrapolation. We further prove that with Markovian data, the asymptotic bias may vanish in several important settings, in which case inference with constant-stepsize LSA iterates is effective even without RR extrapolation.
We conduct extensive experiments to benchmark our procedure against conventional inference approaches. Our results demonstrate superior and robust performance of the constant stepsize paradigm, which enjoys fast convergence, good coverage properties, and easy parameter tuning.

\subsection{Related Work}

Dating back to \cite{Robbins51-Monro-SA}, classical works on stochastic approximation typically assume i.i.d.\ data and a square-summable and diminishing stepsize sequence. Subsequent works propose what is now known as the Polyak-Ruppert iterate averaging \citep{Ruppert88-Avg, polyak90_average} and establish a CLT for the averaged iterates \citep{Polyak92-Avg}. Recent work in \cite{Meyn21_ode} establishes a CLT for scaled iterates for general contractive SA under diminishing stepsize and Markovian data.

Constant-stepsize SA and SGD have attracted growing attention recently. Several works study constant-stepsize LSA with i.i.d.\ data and establish finite-time MSE bounds and a CLT for the average iterates \citep{Lakshminarayanan18-LSA-Constant-iid,Mou20-LSA-iid}. 
A parallel line of work studies SGD  for both convex \citep{Dieuleveut20-bach-SGD} and nonconvex \citep{Yu21-stan-SGD} functions with i.i.d.\ data and identifies an asymptotic bias arising from the nonlinearity of the function.
More recent works study LSA with Markovian data. The works in \cite{srikant-ying19-finite-LSA} and \cite{durmus22-LSA} study convergence bounds for MSE, while the work in \cite{huo2022bias} establishes weak convergence of the LSA iterates and characterizes its asymptotic bias due to Markovian data.

Most related to this paper is a recent line of work on using SGD/SA iterates for statistical inference.
The work in \cite{ChenXi2020} considers SGD with i.i.d.\ data and strongly-convex functions and proposes two covariance matrix estimators. This result is generalized to $\phi$-mixing data in \cite{liu2023-sgd-inference}. The work in \cite{ZhuChenWu-2023-CovEst-SGD} extends the batch-mean estimator in \cite{ChenXi2020} to a fully online version. The work in \cite{LeeLiaoSeoShin2022-SGD-RS} proposes a random scaling covariance estimator for robust online inference with SGD. Subsequent work in \cite{li2023online} investigates online statistical inference using nonlinear SA with Markovian data. The above works all consider diminishing stepsizes. The work in \cite{XieZhang_SAInference_pku} extends the random scaling estimator for inference with SA under constant stepsizes but i.i.d.\ data. The work in \cite{li2017-constantine} studies statistical inference with SGD and i.i.d.\ data, using a small stepsize whose value remains constant throughout the iterations but scales inversely with the total number of iterations. In contrast, we consider constant stepsizes whose values are independent of the total number of iterations and thus substantially larger than the typical stepsize values in \cite{li2017-constantine}.   

RR extrapolation is a classical technique from numerical analysis to improve approximation errors. It has been used in \cite{Dieuleveut20-bach-SGD} for SGD and \cite{huo2022bias} for LSA. See the survey by \cite{Bach2021_RRSurvey} for the use of RR extrapolation in other data science and machine learning problems.

\section{Problem Setup}

In this section, we formally set up the problem and introduce the assumptions.
Let $(x_t)_{t\geq0}$ be a time-homogeneous stochastic process on a Borel state space $\cX$ with stationary distribution $\pi$. Define the target vector $\theta^\ast$  as the solution to the steady-state equation $\E_{x\sim\pi}[A(x)]\theta^\ast + \E_{x\sim\pi}[b(x)]=0,$
where $A:\cX\to\R^{d\times d}$ and $b:\cX\to\R^d$ are deterministic functions on $\cX$.
To approximate $\theta^\ast$, we consider the following linear stochastic approximation iteraion:
\begin{equation}
\label{eq:lsa-iteration}
\theta_{t+1}^{(\alpha)}=\theta_t^{(\alpha)} + \alpha_t\big(A(x_t)\theta_t^{(\alpha)}+b(x_t)\big),\quad t=0,1,\ldots,
\end{equation}
where $(\alpha_t)_{t\ge0}$ is the stepsize sequence. We focus on using a constant stepsize, i.e., $\alpha_t\equiv\alpha$ for all $t\geq0$. (We omit the superscript $(\alpha)$ when the stepsize is clear from the context.)
We emphasize that the constant stepsize considered here is independent of the \emph{total} number of iterations. This is different from the work in \cite{li2017-constantine}, which pre-specifies the total number of iterations $T$ and uses a small fixed stepsize of the form $\alpha_t = T^{-\beta},\forall 0\le t\le T$ for some $\beta>0.$ 

We make the following standard assumptions.
\begin{assumption}
\label{assumption:uniform-ergodic}
    $(x_t)_{t\geq0}$ is a uniformly ergodic Markov chain with transition kernel $P$ and a unique stationary distribution $\pi$. Additionally, the initial state satisfies $x_0\sim\pi$.
\end{assumption}
Assumpiton~\ref{assumption:uniform-ergodic} is common in the literature on Markovian SA \citep{Bhandari21-linear-td, durmus22-LSA, huo2022bias}. Uniform ergodicity ensures that the distribution of $x_t$ converges geometrically to $\pi$ from any initial distribution. For example, all irreducible, aperiodic, and finite state space Markov chains are uniformly ergodic. 

\begin{assumption}
\label{assumption:boundedness}
\label{assumption:hurwitz}
$\Amax:=\sup_{x\in\cX}\|A(x)\|_2\leq 1$ and $\bmax:=\sup_{x\in\cX}\|b(x)\|_2<\infty.$ Moreover, $\BarA:=\E_{x\sim\pi}[A(x)]$ is a Hurwitz matrix.
\end{assumption}
The Hurwitz condition is again standard and ensures the stability of a dynamic system \citep{srikant-ying19-finite-LSA, durmus22-LSA, huo2022bias}. This assumption is satisfied in, e.g., SGD for minimizing strongly convex quadratics and the linear TD algorithm in RL.

\section{Central Limit Theorem}

Our first result is a Central Limit Theorem (CLT) for averaged LSA iterates, which lays the theoretical foundation for the inference procedure developed later. To state the CLT, we recall a known result for constant-stepsize Markovian LSA: the data-iterate pair $(x_t, \theta_t)$ converge weakly to a unique limit distribution, denoted by $(x_\infty,\theta_\infty)\sim\mu$ \citep{huo2022bias}. 
\begin{thm}[CLT]
\label{thm:CLT}
Under Assumptions~\ref{assumption:uniform-ergodic}--\ref{assumption:hurwitz}, there exists a threshold $\alpha_0\in(0,1)$ such that for all $\alpha\in(0,\alpha_0)$, we have
    \begin{equation*}
        \sqrt{T}(\bar\theta_T-\E[\theta_\infty])\overset{\textup{d}}{\rightarrow}\cN(0,\Sigma^\ast),\quad\text{as }T\to\infty,
    \end{equation*}
where $\bar\theta_T:=\frac{1}{T}\sum_{t=0}^{T-1}\theta_k$ and $\Sigma^\ast:=\lim_{T\to\infty}\E_\mu[(\bar\theta_T-\E[\theta_\infty])(\bar\theta_T-\E[\theta_\infty])^\top]$.
\end{thm}

The proof is deferred to Appendix~\ref{sec:proof-clt}.

Theorem~\ref{thm:CLT} extends existing CLT results, which focus on either LSA with i.i.d.\ data \citep{Mou20-LSA-iid, XieZhang_SAInference_pku} or SA with diminishing stepsize \citep{Meyn21_ode}.
When using Markovian data $(x_t)_{t\geq0}$, the iterate sequence $(\theta_t)_{t\geq0}$ is no longer a Markov chain on its own.
Instead, we need to consider the joint process $(x_t,\theta_t)_{t\geq0}$, which is a time-homogeneous Markov chain thanks to the use of a constant stepsize, and build the CLT accordingly. 
Moreover, a number of existing Markov Chain CLT results require a one-step contraction property of the form
$ W_{p,q}(\mu Q,\nu Q)\leq\gamma W_{p,q}(\mu,\nu)$, 
where $Q$ is the Markovian transition kernel, $\gamma<1$ and $W_{p,q}$ is an appropriate Wasserstein distance between distributions \citep{XieZhang_SAInference_pku}.
Our Markov chain $(x_t,\theta_t)$ does not enjoy such a nice one-step contraction property, and hence proving our CLT requires additional work.

\section{Statistical Inference Procedure using LSA}
\label{sec:inference_procedure}

We next present a statistical inference procedure using the averaged LSA iterates with constant stepsize and Markovian data. This procedure can be combined with RR extrapolation to construct confidence intervals for the target vector $\theta^\ast$.

\subsection{Asymptotic Bias and RR Extrapolation}

It has been shown in \cite{huo2022bias} that the asymptotic expectation $\E[\theta_\infty^{(\alpha)}]$ is biased with respect to $\theta^\ast$ and admits expansion: $\E[\theta_\infty^{(\alpha)}]=\theta^\ast + \sum_{i=1}^\infty\alpha^i B^{(i)},$
where $B^{(i)}$ are vectors independent from the stepsize $\alpha$. 
Theorem~\ref{thm:CLT} guarantees that the averaged iterates $\bar{\theta}_T$ are asymptotically normal, but $\E[\theta_\infty^{(\alpha)}] \neq \theta^\ast$. Moreover, the leading term in $\E[\theta_\infty^{(\alpha)}] - \theta^\ast$ scales with $\alpha$. To construct CIs with good coverage properties, it is important to reduce the asymptotic bias.

In light of the bias expansion, we can employ Richardson-Romberg (RR) extrapolation to reduce the asymptotic bias to a higher order polynomial of the stepsize $\alpha$ \citep{Dieuleveut20-bach-SGD, Bach2021_RRSurvey, huo2022bias}. 
Specifically, we run the LSA update~\eqref{eq:lsa-iteration} with $M$ constant stepizes $\cA=\{\alpha_1,\ldots, \alpha_M\}$ and compute a linear combination of the resulting iterates:
\begin{equation*}
\label{eq:rr-iterate}
    \widetilde\theta^\cA_t=\sum_{m=1}^M h_m \theta_t^{(\alpha_m)}.
\end{equation*}
We carefully choose the coefficients $\{h_m\}$ to satisfy
\begin{equation}
\label{eq:RR_coefficients}
\sum_{m=1}^{M}h_{m}=1;\;\;\sum_{m=1}^{M}h_{m}\alpha_{m}^{l}=0,\;l=1,2,\ldots,M-1.
\end{equation}
Using the aforementioned expansion, one sees that the bias $\E[\widetilde\theta^\cA_\infty] - \theta^\ast$ is reduced exponentially in $M$ and now scales with $( \max_{m=1,\ldots, M}\alpha_m)^M$ instead of $\alpha_m$.

\subsection{Inference Procedure}
\label{sec:inference-proc}
We now describe the inference procedure, which follows from the procedure in \cite{li2017-constantine} originally designed for i.i.d.\ data and a small stepsize.

\paragraph*{Point Estimation and Batching}

Given a trajectory of $(x_t)_{t\geq0}$ sampled from a Markov chain, we run LSA with constant stepsize $\alpha$ and obtain iterates $(\theta_t^{(\alpha)})_{t\geq0}$. The first $b$ iterates $(\theta^{(\alpha)}_t)_{t=0}^{b-1}$ are considered as initial burn-in and are not used in the inference procedure. For the remaining iterates, we divide them equally into $K$ batches of size $n$. Within each batch, we discard the first $n_0(\geq0)$ iterates and compute the average of the remaining iterates:
\begin{equation*}
    \underbrace{\theta_0^{(\alpha)},
    \ldots, \theta_{b-1}^{(\alpha)}}_\text{burn in, discarded}, 
    \overbrace{\underbrace{\theta_b^{(\alpha)},  \ldots, \theta_{b+n_0-1}^{(\alpha)}}_\text{discard},\underbrace{\theta_{b+n_0}^{(\alpha)},\ldots,\theta_{b+n-1}^{(\alpha)}}_\text{compute average $\bar\theta_1^{(\alpha)}$}}^\text{1st batch},
    \overbrace{\underbrace{\theta_{b+n}^{(\alpha)},  \ldots, \theta_{(b+n)+n_0-1}^{(\alpha)}}_\text{discard},\underbrace{\theta_{(b+n)+n_0}^{(\alpha)},\ldots,\theta_{b+2n-1}^{(\alpha)}}_\text{compute average $\bar\theta_2^{(\alpha)}$}}^\text{2nd batch},\ldots
\end{equation*}

Hence, for the $k$-th batch, we compute the point estimator $\bar\theta_k^{(\alpha)}=\frac{1}{n-n_0}\sum_{l=b+(n-1)k+n_0}^{b+nk-1}\theta_l^{(\alpha)}$. As such, we obtain $K$ batch-mean estimators $\{\bar\theta^{(\alpha)}_{k}\}$, which will be used for  statistical inference. Note that we only need to save the running average of each batch-mean estimator without the necessity to store the entire trajectory $(\theta_t^{(\alpha)})_{t\geq0}$.

Before delving into the construction of CIs, we briefly remark on several design choices. The initial $b$ iterates are considered as burn-in and are omitted, as these iterates are far away from stationarity, and thus may have substantial optimization errors. The first $n_0$ iterates of each batch are also discarded, to reduce the correlation of the remaining iterates across batches, i.e., in the order of $\exp(-\alpha n_0)$.

\paragraph*{Confidence Interval Construction}
\label{sec:batch-mean}
Now that the CLT for the average iterates of Markovian LSA with constant stepsizes has been established, we construct estimators for $\E[\theta_\infty^{(\alpha)}]$ and $\Sigma^\ast$ and subsequently build CIs for $\E[\theta_\infty]$.

With the $K$ batch-mean estimators $\{\bar\theta^{(\alpha)}_{k}\}$ for $k=1,\ldots, K$, we compute batch-mean estimators as
\begin{equation*}
    \bar\theta^{(\alpha)}=\frac{1}{K}\sum_{k=1}^K\bar\theta_k^{(\alpha)}.
\end{equation*}

For variance estimation, we adapt an estimator that has been considered in \cite{FlegJone2010-VarEst, ChenXi2020, XieZhang_SAInference_pku} to our problem. Given $\{\bar\theta_k^{(\alpha)}\}$, the variance estimator is computed as
\begin{equation*}
    \hat{\Sigma}^{(\alpha)}=\frac{(n-n_0)}{K}\sum_{i=1}^K\Big(\bar{\theta}^{(\alpha)}_k - \bar\theta^{(\alpha)}\Big)\Big(\bar{\theta}^{(\alpha)}_k - \bar\theta^{(\alpha)}\Big)^\top,
\end{equation*}
It has been shown in \cite{FlegJone2010-VarEst} that $\hat\Sigma$ is a consistent estimator of $\Sigma^\ast$ in Theorem~\ref{thm:CLT} as $n,K\to\infty$, 

Hence, for inference with LSA with stepsize $\alpha$, for $q\in(0,1)$, we construct
the $(1-q)100\%$-confidence interval for the $i$-th coordinate of $\E[\theta_\infty]$ as
\begin{equation*}
    \Bigg[\bar\theta^{(\alpha)}_i-z_{1-\frac{q}{2}}\sqrt{\frac{\hat\Sigma^{(\alpha)}_{i,i}}{K(n-n_0)}}, \,\bar\theta^{(\alpha)}_i+z_{1-\frac{q}{2}}\sqrt{\frac{\hat\Sigma^{(\alpha)}_{i,i}}{K(n-n_0)}}\Bigg].
\end{equation*}
In subsequent experiments, we focus on 95\% CIs.

\subsection{Combining with RR Extrapolation} 
Next, we apply RR extrapolation in addition to the above-delineated procedure to construct confidence intervals that have better coverage properties of $\theta^\ast$.

To apply RR extrapolation, we first select a set of $M$ distinct stepsizes, i.e., $\cA=\{\alpha_1,\alpha_2,\ldots, \alpha_M\}$ and run LSA iterates with these stepsizes simultaneously by using the same underlying data stream $(x_k)_{k\geq0}$. 
For each trajectory of iterates $(\theta^{(\alpha_m)}_t)_{t\geq0}$, we follow the inference procedure and obtain $K$ batch-mean estimators $\{\bar\theta^{(\alpha_m)}_{k}\}_{k=1}^K$. To obtain the RR extrapolated estimator, we linearly combine the $k$-th estimates across the $M$ trajectories and obtain $\widetilde{\theta}^\cA_k=\sum_{m=1}^Mh_m\bar\theta_k^{(\alpha_m)}$, with $\{h_m\}$ computed according to~\eqref{eq:RR_coefficients}. 
We then conduct statistical inference using the iterates $\{\widetilde{\theta}^\cA_k\}$ and build confidence intervals following the similar methodology described above for a single trajectory.

\section{Theoretical Guarantees}
Next, we provide additional theoretical analysis that helps us achieve better statistical inference performance with RR extrapolated iterates. 

\subsection{Stepsize Selection in RR Extrapolation}

As we discussed, one way to improve the coverage probability of $\theta^\ast$ is to reduce the bias via RR extrapolation. 
However, RR extrapolation does not come for free, as the coefficients $\{h_m\}$ solving~\eqref{eq:RR_coefficients} may blow up, thus resulting in large variance and offsetting the benefits of bias reduction.
$\{h_m\}$ values are uniquely determined by inverting a Vandermonde matrix, which is infamous for being ill-conditioned when the ``roots" $\{\alpha_m\}$ are positive real numbers. Therefore, when employing RR extrapolation, we need to carefully select $M$ and $\{\alpha_m\}$ to maximize the benefits of bias reduction.

We study two stepsize selection schedules, namely geometric decaying and equidistant sequences. 
We assume that $\alpha_m$ decreases in value as $m$ increases. 
We establish an upper bound to the variance of $\theta_\infty$ in each stepsize regime, which would offer some insight and guidance on stepsize selection.

The geometric decay schedule is not a conventional choice in numerical analysis, the field from which RR extrapolation originates, but it is frequently employed in machine learning. 
\begin{proposition}
\label{prop:rr-geo}
    Given unique stepsizes $\cA=\{\alpha_1,\ldots, \alpha_M\}$. Assume $\alpha_1<1$ and $\alpha_m=\alpha_1 / c^{m-1}$ with $c\geq 2$. We observe the following properties.
    \begin{enumerate}
        \item $|h_k|\leq h_{\max}(c)=\exp\Big(\frac{2}{c-1}\Big)$.
        \item $\var(\widetilde\theta^\cA_\infty)=\bigO\Big(c\cdot\exp(16\, c^{-1/2})\Big)$.
    \end{enumerate}
\end{proposition}

Please refer to Appendix~\ref{sec:proof-var-geo} for the proof of the proposition.

It is noteworthy that the variance upper bound for geometric decaying stepsizes does not depend on the number of stepsizes $M$ used in the extrapolation, suggesting that the variance will not blow up as $M\to\infty$. Nonetheless, one problem with geometric decay is that the stepsize would decay too quickly to near zero as $M$ increases. As such, RR iterates with a small constant stepsize mixes slowly and may need a much longer trajectory for the iterates to converge.

The equidistant decay schedule has been studied in numerical analysis \citep{Gautschi90}. We investigate the behavior of a generic equidistant decay schedule in RR exploration.
\begin{proposition}
\label{prop:rr-eqd}
    Given unique stepsizes $\cA=\{\alpha_1,\ldots, \alpha_M\}$. Assume $a+b<1$ and $\alpha_m=(a+b)-\frac{b(m-1)}{M-1}$ for $m=1,\ldots, M$. We observe the following properties.
    \begin{enumerate}
        \item $|h_m|=\binom{M}{m}\cdot\frac{bm}{a(M-1)+b(M-m)}\cdot\Big(\prod_{l=1}^M\frac{a(M-1)+b(l-1)}{bl}\Big)$.
        \item $\var(\widetilde\theta^\cA_\infty)=\bigO((2M/b)^{2M})$.
    \end{enumerate}
\end{proposition}

Please refer to Appendix~\ref{sec:proof-var-eqd} for the proof of the proposition.

When the stepsize sequence decays in equidistance, stepsizes do not decrease as quickly to zero based on the choice of $a,b$. However, the upper bound of variance now is of order $M^M$, which suggests that the variance could blow up quickly as $M$ increases, potentially offsetting the benefits from reduced bias.

\subsection{Zero Bias Special Cases}
\label{sec:zero-special-cases}

As discussed earlier, RR extrapolation reduces the asymptotic bias and hence increases the asymptotic coverage probability of $\theta^\ast$ based on the CI constructed for $\E[\theta_\infty]$. Hence, the Markovian underlying data and the presence of asymptotic bias should not discourage one from using constant stepsize LSA iterates for statistical inference.

In this section, we would like to highlight that Markovian data need not be a sufficient condition for the presence of asymptotic bias in LSA. That is, there are several commonly observed scenarios where no asymptotic bias is present, even when the underlying data is Markovian.

\subsubsection{Independent Multiplicative Noise}

In this scenario, we expand our Markov chain state space to incorporate an independent bounded zero-mean random variable, i.e., $x_t=(s_t,\xi_t)$, where $(s_t)_{t\geq0}$ is the uniformly ergodic Markov chain and $(\xi_t)_{t\geq0}$ is uniformly bounded, i.e., $\|\xi_t\|\leq u$, and i.i.d.\ sampled from distribution $\xi$ with $\E[\xi]=0$. Consider the following Markovian LSA updates, 
\begin{equation}
\label{eq:add-mkv-noise}
    \theta_{t+1}=\theta_t+\alpha\Big((\BarA+\xi_t)\theta_t + b(s_t)\Big),
\end{equation}
where $A(s_t,\xi_t)=\BarA+\xi_t$ and $b(s_t,\xi_t)=b(s_t)$.
It can be easily verified that the above LSA iteration satisfies the required Assumptions~\ref{assumption:uniform-ergodic}--\ref{assumption:hurwitz}.

\begin{proposition}
\label{prop:lsa-indep-mul}
    Consider the LSA iteration of \eqref{eq:add-mkv-noise}, there exists some threshold $\alpha_0\in[0,1)$ such that $\forall \alpha\in[0,\alpha_0)$, $(x_t,\theta_t)_{t\geq0}$ converges weakly to a unique stationary distribution. Moreover, $\E[\theta_\infty]=\theta^\ast$.
\end{proposition}

Please refer to Appendix~\ref{sec:proof-mkv-add} for the proof of the proposition.

The above proposition implies that when the Markovian noise is only additive, the limiting expectation of the iterates converges to $\theta^\ast$. Hence, our CLT on $\E[\theta_\infty]$ allows us to construct CI and perform statistical inference directly on $\theta^\ast$. 

\subsubsection{Independent Additive Noise in Linear Regression}

We consider an independent additive noise setting in linear regression, which has been previously discussed in \cite{guy2020} and \cite{huo2022bias}. 

In this linear regression problem, independent observations $(s_t)_{t\geq0}$ are sequentially generated from a uniformly ergodic Markov chain with stationary distribution $\nu$ and $\E_\nu[s_ts_t^\top]$ is positive definite. $y_t=s_t^\top w^\ast+\epsilon_t$, where $\epsilon_t$ is an i.i.d.\ zero-mean noise with variance $\sigma^2_\epsilon$. 
The SGD iterates to estimate $w^\ast$ are updated in the following fashion, $  w_{t+1}=w_t-\alpha s_t\Big(\langle w_t,s_t\rangle - y_t\Big).$
Casting the problem under the LSA framework, we have $w_{t+1}=w_t+\alpha(A(x_t)w_t+b(x_t))$ with $x_t=(s_t,\epsilon_t)$, $A(x_t)=-s_ts_t^\top$ and $b(x_t)=s_t(s_t^\top w^\ast+\epsilon_t)$.
It has been shown in \cite{huo2022bias} that LSA under constant stepsizes has zero asymptotic bias in this setup.

\subsubsection{Realizable Linear-TD Learning in RL}
In this section, we specialize our discussion on linear-TD learning in RL.

TD learning algorithm~\citep{Sutton1988-td} is an important algorithm in RL for policy evaluation. Potentially equipped with linear function approximation, it is a special case of LSA.

We model the linear-TD problem with a Markov Reward Problem (MRP) $(\cS, P^\cS, r,\gamma)$. The value function $V:\cS\to\R$ is approximated by a linear function $V(s)=\E[\sum_{t=0}^\infty\gamma^tr(s_t)|s_0=s]\approx\phi(s)^\top\theta$, where $\phi:\cS\to\R^d$ is a known feature map of finite rank $d$ and $\theta$ is the unknown weight vector to be estimated. The feature map is normalized such that $\phi_{\max}:=\sup_{s\in\cS}\|\phi(s)\|_2\leq\frac{1}{\sqrt{1+\gamma}}$.

We consider the semi-simulator regime, in which the iterates are updated in the following fashion,
\begin{equation}
    \label{eq:linear-td}
    \theta_{t+1}=\theta_t+\alpha\Big(r(s_t)+\gamma\phi(s_t^\text{next})^\top\phi_t-\phi(s_t)^\top\theta_t\Big)\phi(s_t),
\end{equation}
where $(s_{t})_{t\ge0}$ is a Markov chain and $s_{t}^{\text{next}}$ is sampled independently from $P^{\mS}(s_{t},\cdot)$ conditioned on $s_{t}$.

\begin{proposition}
\label{prop:td-realize}
Assuming $(s_k)_{k\geq0}$ is a uniformly ergodic Markov chain on state space $\cS$ with $s_0\sim \pi^\cS$. Assuming the linear-TD is realizable, i.e., $\exists v\in\R^d$ such that $V(s)=\phi(s)^\top v$ for all $s\in\cS$. The linear-TD iterates of \eqref{eq:linear-td} converge without asymptotic bias, i.e., $\E[\theta_\infty]=\E[\theta^\ast]$.
\end{proposition}

Please refer to Appendix~\ref{sec:proof-realize} for the proof of the proposition.

The realizability assumption is not particularly restricting, as there are a number of RL problems that satisfy this assumption, such as TD learning on tabular MDP or linear MDP.
Hence, if we have such a structured linear-TD problem, we could use the iterates obtained in the semi-simulator setting to perform statistical inference on the weight vector $\theta$ estimation and subsequently the value function directly.

\section{Numerical Experiments}

We conduct extensive numerical experiments to examine our proposed inference procedure and stepsize selection guidelines in RR extrapolation. We present our main results in this section. Additional sets of experiments and detailed experiment designs are included in Appendix~\ref{sec:exp-details}.

\subsection{Inference Performance Comparison}

We examine the empirical performance of our proposed inference procedure with constant stepsizes and RR extrapolation. We consider LSA problems in dimension $d=5$ for a finite state, irreducible, and aperiodic Markov chain with $N=10$ states. We generate the transition probability $P$, and the functions $A$ and $b$ randomly; see the appendix for details. 
We construct CIs using the LSA iterates. We examine the performance from three aspects, namely the $\ell_2$ error of the point estimate to the target vector, i.e.,$\|\bar\theta-\theta^\ast\|_2$, the coordinate-wise CI width and coverage probabilities. 
For more details on setting up the LSA problem, please refer to Appendix~\ref{sec:lsa-exp-details}.

We mainly study under constant stepsizes $\alpha=0.2$ and $\alpha=0.02$. 
The two choices of constant stepsizes are of different scales, allowing us to demonstrate extreme effects in convergence and inference performance. RR extrapolation is conducted using the two constant stepsizes. For comparison, we also consider diminishing stepsizes with initial stepsize $\alpha_0$ being 0.2 and 0.02 respectively, and a diminishing stepsize  $\alpha_k = \alpha_0 k^{-0.5}$ for $k\geq1$. The diminishing rate $k^{-0.5}$ is chosen as it is on the boundary of square-summable assumption and has often been observed with the best empirical performance among $k^{-\beta}$ with $\beta\in[0.5,1]$. 
For more details on the inference procedure with diminishing stepsize LSA iterates, please refer to Appendix~\ref{sec:inf-proce-dimin}.

\subsubsection{Baseline Comparison for i.i.d.\ LSA}
We first conduct a cross-study to examine the performance of inference with constant and diminishing stepsize under i.i.d.\ data. We notice that constant stepsizes slightly outperform the $0.2  k^{-0.5}$ diminishing stepsize. Please refer to Appendix~\ref{sec:baseline-iid} for detailed experiment design and results.

\subsubsection{Performance Comparison for Markovian LSA}
We examine 100 different LSA problems of the same dimension $|\cX|=10$ and $d=5$. For each LSA setup, the parameters $(P,A,b)$ are generated randomly, and we simulate 100 trajectories $(x_k)_{k\ge0}$ of length $10^5$, run LSA iterates with the above-described stepsize regimes, and perform inference. We summarize the distribution of the performance across the 100 cases in Table~\ref{tab:quantile table mkv}.

With Markovian data, the LSA with constant stepsize converges with nonzero asymptotic bias. The presence of bias is rather obvious when $\alpha=0.2$, evident from the large $\ell_2$ error. Hence, CIs constructed with $\alpha=0.2$ iterates have the worst coverage probabilities. Reducing the stepsize to $\alpha=0.02$ will reduce the asymptotic bias, and hence CIs constructed with $\alpha=0.02$ iterates have significantly better performance. 

When the RR extrapolation technique is employed, the confidence intervals constructed enjoy the best coverage properties, with the smallest $\ell_2$ error, comparable CI width, and higher coverage probabilities of $\theta^\ast$. Moreover, the median coverage probability is around the targeted 95\%. 

\begin{table}[ht]
    \centering
    \begin{tabular}{ |c|c|c|c|c|c|c| } 
        \hline
        \multicolumn{2}{|c|}{\multirow{2}{*}{Percentile} }&\multicolumn{5}{|c|}{Comparison table} \\
        \cline{3-7}
        \multicolumn{2}{|c|}{}& 0.2 & 0.02 & RR & $0.2/\sqrt{k}$ & $0.02/\sqrt{k}$\\ 
        \hline
        \multirow{3}{*}{10} & $\ell_2$ & 3.60 & 1.01 & 0.92 &  0.87 & 0.90\\ 
        & CI  &1.44 & 1.30 & 1.31 & 1.21 & 0.87\\ 
        & Cov & 0 & 72 & 90 & 86 & 62\\ 
        \hline
        \multirow{3}{*}{25} & $\ell_2$ & 6.05 & 1.25 & 1.08 & 1.06 & 1.11\\ 
        & CI  & 1.87 & 1.68 & 1.70 & 1.52 & 1.20\\ 
        & Cov & 0 & 82 & 91 &88 & 70\\ 
        \hline
        \multirow{3}{*}{50} & $\ell_2$ & 8.12 & 1.59 & 1.32 & 1.32 & 1.42\\ 
        & CI  & 2.70 & 2.38 & 2.41 & 2.14 & 1.51\\ 
        & Cov & 11 & 90 & 94 & 91&76\\ 
        \hline
         \multirow{3}{*}{75} & $\ell_2$ & 14.82 & 2.39 & 1.90 &1.85 & 2.13\\ 
        & CI  & 3.95 & 3.40 & 3.47& 3.05 & 2.50\\ 
        & Cov & 66 & 93 & 95 & 94& 83\\ 
        \hline
        \multirow{3}{*}{90} & $\ell_2$ & 25.53 & 4.20 & 4.14 & 3.44& 6.92\\ 
        & CI & 10.49 & 6.31 & 8.91 & 5.21 & 4.82\\ 
        & Cov & 92 & 95 &97 & 96& 90\\ 
        \hline
    \end{tabular}
    \caption{Inference comparison of different stepsize regimes. 
    $\ell_2$ denotes the $\ell_2$ error of point estimates and is of unit $10^{-3}$. ``CI" refers to the CI width and is also of unit $10^{-3}$. ``Cov" represents the coverage probability. Both the CI width and coverage probability are for the 1-st coordinate estimate.
    }
    \label{tab:quantile table mkv}
\end{table}

\subsubsection{Batch Number Selection}

We next inspect the impact of batch number selection, i.e., the value $K$ in Section~\ref{sec:inference-proc}. We focus on one specific LSA problem with $|\cX|=10$ and $d=5$ Markovian data. The trajectory length is set at $10^6$ and the number of batches varies. For each combination of batch number and stepsize, we run 500 independent runs and record the mean and the standard error in Table~\ref{tab:batch_label}.

For constant stepsizes, as we compare across different rows, the mean estimates are at the same scale, and not influenced much by the choice of batch number.
However, for diminishing stepsizes, as the batch number increases beyond the recommended level proposed in \cite{ChenXi2020}, the CI widths decrease quickly, which drastically reduces the coverage probability and implies an underestimation of the variance. As the stepsize decays in the diminishing stepsize sequence, the iterates become increasingly correlated, and hence a longer batch size is needed to overcome the strong correlation. Therefore, the number of batches used in diminishing stepsize regime cannot arbitrarily increase. In contrast, iterates under constant stepsize mixes at the same rate, and hence is more robust to the number of batches used in inference. 

\begin{table}[tb]
    \centering
    \begin{tabular}{ |c|c|c|c|c|c|c| } 
        \hline
        \multicolumn{2}{|c|}{\multirow{2}{*}{Batch No.} }&\multicolumn{5}{|c|}{Comparison table} \\
        \cline{3-7}
        \multicolumn{2}{|c|}{}& 0.2 & 0.02 & RR & $0.2/\sqrt{k}$ & $0.02/\sqrt{k}$\\ 
        \hline
        \multirow{6}{*}{50} & \multirow{2}{*}{$\ell_2$} & 7.79 & 1.03 & 0.74 & 0.71 & 0.73 \\
        && (0.02) & (0.02) & (0.02) & (0.02)& (0.02)\\
        & \multirow{2}{*}{CI} & 1.52 & 1.37 & 1.37 & 1.30 & 1.01 \\ 
        && (0.01) & (0.01) & (0.01) & (0.01)& (0.01)\\ 
        & \multirow{2}{*}{Cov} & 0 & 78.6 &92.8  & 93.0 & 81.6 \\ 
        && (0) & (1.8) &  (1.1) & (1.1)& (1.7)\\ 
        \hline
        \multirow{6}{*}{100} & \multirow{2}{*}{$\ell_2$} & 7.80  & 1.01  & 0.70  & 0.69 & 0.72 \\ 
        &&(0.02) & (0.02)& (0.02)& (0.02)&(0.02)\\
        & \multirow{2}{*}{CI} & 1.53 & 1.37 & 1.38  & 1.34 & 0.81 \\ 
        &&(0.01) & (0.00)&(0.00)&(0.00)&(0.00)\\
        & \multirow{2}{*}{Cov} & 0 & 80.6  &94.4   & 95.0 & 71.2 \\ 
        &&(0) &(1.8) &(1.0) &(1.0)&(2.0)\\
        \hline
        \multirow{6}{*}{500} & \multirow{2}{*}{$\ell_2$} & 7.80  & 1.03  & 0.72 & 0.70& 0.70 \\ 
        &&(0.02)& (0.02) &(0.02) & (0.02) & (0.02)\\
        & \multirow{2}{*}{CI} & 1.54 & 1.38 & 1.39 & 1.06 & 0.40\\ 
        &&(0.00) &(0.00) &(0.00) &(0.00)& (0.00)\\
        & \multirow{2}{*}{Cov} & 0 & 80.8 &94.2  & 88.8&42.2 \\ 
        &&(0) &(1.8) & (1.0) & (1.4) &(2.2)\\
        \hline
         \multirow{6}{*}{$1000$} & \multirow{2}{*}{$\ell_2$} & 7.79 & 1.01 & 0.73 & 0.69 & 0.74 \\ 
         && (0.02) & (0.02) & (0.02) & (0.02)& (0.03)\\
        & \multirow{2}{*}{CI} & 1.54 & 1.38 & 1.39& 0.82& 0.29 \\ 
        && (0.00) & (0.00) & (0.00) &  (0.00) &(0.00)\\
        & \multirow{2}{*}{Cov} & 0  & 83.0 &94.2 & 75.4& 30.4 \\ 
        &&(0) &(1.7) & (1.0) &  (1.9) &(2.1)\\
        \hline
    \end{tabular}
    \caption{Inference comparison of different batch numbers.
    $\ell_2$ and ``CI" values are of unit $10^{-3}$. Both CI width and coverage probability are for the 1-st coordinate estimate.}
    \label{tab:batch_label}
\end{table}

\subsubsection{Trajectory Length}

We next investigate the impact of increasing trajectory length with a range of stepsizes to better visualize the trend. For each trajectory length, we fix the batch number as $T^{0.3}$ as recommended in~\cite{ChenXi2020}. The statistics are in Table~\ref{tab:traj_table}. 

We observe that inference with RR iterates consistently presenting the best results.
When the trajectory length is at $10^3$, the iterates under various stepsize regimes would still be distant from $\theta^\ast$, which explains the mediocre coverage of $\theta^\ast$. 
As the trajectory length increases, the iterates gradually approach stationarity. Hence, the inference performance deteriorates for $\alpha=0.2$ with 0 coverage probability, as the iterates converge to $\E[\theta_\infty]$, which is asymptotically biased from $\theta^\ast$. 
However, the iterates under diminishing stepsize converge closer to $\theta^\ast$, as the trajectory length is longer. Therefore, the inference performance with diminishing stepsize improves.
Inference with RR extrapolated iterates generally gives satisfactory performance. Inference with constant stepsize might be especially useful when the simulation trajectory length budget is limited, as diminishing stepsize iterates struggle to output good inference results.

\begin{table*}[ht]
    \centering
        \begin{tabular}{ |c|c|c|c|c|c|c|c|c|c| } 
        \hline
        \multicolumn{2}{|c|}{\multirow{2}{*}{Length} }&\multicolumn{8}{|c|}{Comparison table} \\
        \cline{3-10}
        \multicolumn{2}{|c|}{}& 0.2 & 0.15 & 0.1 & 0.05 & 0.02 & RR & $0.2/\sqrt{k}$ & $0.02/\sqrt{k}$\\ 
        \hline
        \multirow{6}{*}{$10^3$} & \multirow{2}{*}{$\ell_2$} & 2.84 &2.55 & 2.48 & 2.35 & 2.28 & 2.28 & 2.36 & 4.28 \\ 
        &&(0.06) & (0.06) & (0.05) & (0.05) & (0.05) & (0.05) & (0.05) & (0.07)\\
        & \multirow{2}{*}{CI} & 4.16 & 3.91 & 3.81 & 3.68 & 3.71 & 3.76 & 3.10 & 1.01 \\ 
        && (0.07) & (0.06) & (0.06) & (0.06) & (0.06) & (0.06) & (0.05) & (0.02)\\
        & \multirow{2}{*}{Cov} & 82.2 & 84.0 & 83.0 & 83.6 & 83.6 & 83.4 & 76.8 & 24.8 \\ 
        && (1.7) & (1.6) & (1.7) & (1.7) & (1.7) & (1.7) & (1.9)& (1.9)\\
        \hline
        \multirow{6}{*}{$10^4$} & \multirow{2}{*}{$\ell_2$} & 1.15 & 0.99 & 0.89 & 0.79 & 0.73 & 0.72 & 0.73 & 1.14\\ 
        && (0.02) & (0.02) & (0.02) & (0.02) & (0.02) & (0.02) & (0.02)& (0.02)\\
        & \multirow{2}{*}{CI} & 1.44 & 1.38 & 1.34 & 1.31 & 1.29 & 1.30 & 1.12 & 0.67 \\ 
        && (0.02) & (0.01) & (0.01) & (0.01) & (0.01) & (0.01) & (0.01)& (0.01)\\
        & \multirow{2}{*}{Cov} & 75.2 & 81.8 & 85.2 & 88.8 & 90.0 & 89.2 & 85.4 & 56.0 \\ 
        && (1.9) & (1.7) & (1.6) & (1.4) & (1.3) & (1.4) & (1.6)& (2.2)\\
        \hline
        \multirow{6}{*}{$10^5$} & \multirow{2}{*}{$\ell_2$} & 0.82 & 0.61 & 0.44 & 0.29 & 0.24 & 0.23 & 0.22 & 0.24 \\ 
        && (0.01) & (0.01) & (0.01) & (0.01) & (0.01) & (0.01) & (0.01)& (0.01)\\
        & \multirow{2}{*}{CI} & 0.46 & 0.45 & 0.43  & 0.42 & 0.41 & 0.41 & 0.38 & 0.29 \\ 
        && (0.00) & (0.01) & (0.01) & (0.01) & (0.00) & (0.00) & (0.00)& (0.00)\\
        & \multirow{2}{*}{Cov} & 0.04 & 22.2 & 56.4 & 83.2 & 88.2 & 91.2 & 90.0 & 79.2 \\ 
        && (0.9) & (1.8) & (2.2) & (1.7) & (1.4) & (1.3) & (1.3)& (1.8)\\
        \hline
        \multirow{6}{*}{$10^6$} &  $\ell_2$ & 0.81 & 0.57 & 0.38 & 0.20 & 0.00 & 0.00 & 0.00 & 0.00 \\ 
        && (0.00) & (0.00) & (0.00) & (0.00) & (0.00) & (0.00) & (0.00) & (0.00)\\
        & CI & 0.15 & 0.15 & 0.14 & 0.14 & 0.13 & 0.13 & 0.13 & 0.11 \\ 
        && (0.00) & (0.00) & (0.00) & (0.00) &(0.00) & (0.00) & (0.00)& (0.00)\\
        & Cov & 0 & 0 & 0 & 19.0 & 80.2 & 95.2 & 92.8 & 90.2 \\ 
        && (0) &(0) & (0) & (1.8) & (1.8) & (1.0) & (1.2)& (1.3)\\
        \hline
    \end{tabular}
    \caption{Inference comparison of different trajectory lengths.
    $\ell_2$ and ``CI" values are of unit $10^{-2}$. Both CI width and coverage probability are for the 1-st coordinate estimate.}
    \label{tab:traj_table}
\end{table*}

\subsubsection{Comparison Against Bootstrapping}
In all previous experiments, we are comparing inference using SA iterates under different stepsize regimes. Here, we compare to bootstrapping, a more classical approach to statistical inference. A detailed experimental design on bootstrapping inference can be found in Appendix~\ref{sec:bootstrap-details}.

While constant stepsize together with RR obtains $95.2\%$ coverage with $\ell_2$ error $1.0\times 10^{-5}$ and CI width $0.00134$, bootstrapping achieves $93.4\%$ coverage with $\ell_2$ error $2.0\times 10^{-3}$ and CI width $0.00411$. The large $\ell_2$ error and wide confidence interval of bootstrapping suggest that it may not be able to handle correlated data efficiently.

Bootstrapping requires the entire data set to be stored, requiring $\bigO(n^d)$ memory, whereas inference with SA iterates can accommodate online data, hence $\bigO(d)$ memory. Inference with SA iterates only needs first-order information, while bootstrapping may need higher-order information, making it potentially computationally prohibitive.

\subsubsection{Extending to Nonlinear SA}
Lastly, we examine the performance of our proposed technique in an example of logistic regression with unbounded Markovian data, to demonstrate that our proposed technique is robust in a wide range of settings, even if not currently addressed by our theory. 

The categorical data $(x_t,y_t)$ arrives online, with $x_{t}$ from a $2$-dimensional Gaussian AR(1) process, and $y_t\sim\text{Bernoulli}(\frac{1}{1+e^{-w_\ast x}}))$ with $w_\ast\in\R^2$ a random unit vector. 
This setting is non-Hurwitz, non-linear, and with an unbounded underlying Markov chain, but our proposed procedure still exhibits similar performance as seen in Table~\ref{tab:traj_table_reg}.

\begin{table}[t]
    \centering
    \begin{tabular}{ |c|c|c|c|c|c|c| } 
        \hline
        \multicolumn{2}{|c|}{\multirow{2}{*}{Length} }&\multicolumn{5}{|c|}{Comparison table} \\
        \cline{3-7}
        \multicolumn{2}{|c|}{}& 0.2 &  0.02 & RR & $0.2/\sqrt{k}$ & $0.02/\sqrt{k}$\\ 
        \hline
        \multirow{6}{*}{$10^3$} & \multirow{2}{*}{$\ell_2$} & 11.63 & 21.53 & 24.4 & 20.89 & 41.14\\ 
       & & (0.28) & (0.26)  & (0.25)  & (0.28)  & (0.25)\\ 
        & \multirow{2}{*}{CI} & 24.01  & 34.87 & 38.07 & 15.17  & 16.31\\ 
         & & (0.40) &(0.27) & (0.26) & (0.20) & (0.12)\\ 
        &  \multirow{2}{*}{Cov} & 70.2 & 34.0 & 26.4 &3.6 & 0\\ 
         & & (2.0)  & (2.1) & (2.0)  & (0.8) & (0)\\ 
         \hline
        \multirow{6}{*}{$10^4$} & \multirow{2}{*}{$\ell_2$} & 8.35 & 2.87 & 3.27 & 3.65 & 10.64\\ 
        & & (0.11)& (0.07) & (0.08)& (0.09) & (0.11)\\ 
        & \multirow{2}{*}{CI} & 9.27 & 9.37 & 9.79 & 5.84 & 10.23\\ 
        & & (0.10) & (0.09) & (0.09)& (0.07)  & (0.07)\\ 
        &  \multirow{2}{*}{Cov} & 12.0 & 89.8 & 84.8 & 53.6 & 1.4\\ 
        & & (1.5) & (1.4) & (1.6)& (2.2) &  (0.5)\\ 
        \hline
        \multirow{6}{*}{$10^5$} & \multirow{2}{*}{$\ell_2$} & 8.14&1.10 & 0.88& 0.90  & 1.40\\ 
        & & (0.04)& (0.03) & (0.02)& (0.02) & (0.03) \\ 
        & \multirow{2}{*}{CI} & 3.09 & 2.83 & 2.82 & 1.83  & 2.74\\ 
        & & (0.03) & (0.02) & (0.02)& (0.02) & (0.02)\\ 
        &  \multirow{2}{*}{Cov} & 0 &  80.2 & 90.2 & 73.0 &59.6\\ 
        & & (0) & (1.8) & (1.3)& (2.0) & (2.2) \\ 
        \hline
        \multirow{6}{*}{$10^6$} &  \multirow{2}{*}{$\ell_2$} &8.11  & 0.83  & 0.28 & 0.27& 0.28\\ 
            & & (0.01) & (0.01) & (0.01)& (0.01)& (0.01)\\ 
            & \multirow{2}{*}{CI} & 1.00  & 0.94 & 0.93 & 0.68 & 0.66 \\ 
            & & (0.01) & (0.01) & (0.01)& (0.00)& (0.00) \\ 
            &  \multirow{2}{*}{Cov} & 0  &8.0 & 95.0 & 82.6 &78.0 \\ 
            & & (0) & (1.2) & (1.0)& (1.7) &(1.9) \\ 
        \hline
    \end{tabular}
\caption{Inference comparison of different trajectory lengths.
$\ell_2$ and ``CI" values are of unit $10^{-2}$. Both CI width and coverage probability are for the 1-st coordinate estimate.}
\label{tab:traj_table_reg}
\end{table}

\subsection{RR Stepsize Selection}

Next, we numerically examine the impacts of different stepsize selection decisions in RR extrapolation. 

\subsubsection{Comparing Different Regimes}
We first compare the geometric and equidistant regimes. To ensure a fair comparison, we keep the range of the stepsize selection the same across the two regimes, so the difference would come solely from how the stepsize decays within the range. The range is determined by a dyadic decaying stepsize sequence, i.e. the smallest stepsize across the two regimes is fixed at $\alpha_M=\alpha_1/2^{(M-1)}$.
Under this setup, when the RR order is 2, there is no difference between these two schedules, so we start the comparison from the extrapolation with 3 stepsizes.

As depicted in Figure~\ref{fig:rr-eqd-geo}, when the order of extrapolation increases to around 5, the benefits of RR extrapolation have been mostly exploited. The difference between the two decaying schedules is more significant when the order is low.  

\subsubsection{Comparing Spacing in the Equidistant Regime}

We study the impact of decay distance in the equidistant decaying scheme, i.e., varying the value of $b$ in $\alpha_m=(a+b)-b\frac{(m-1)}{M-1}$. 
Specifically, we keep the $a+b$ value constant, and test with different $b/(M-1)\in\{0.01, 0.02, 0.03, 0.04\}$. 

As shown in Figure~\ref{fig:rr-eqd}, the smaller $b/(M-1)$ is, the worse the RR extrapolation performance is. This phenomenon is in line with Proposition~\ref{prop:rr-eqd}, as $b/(M-1)$ inversely scales with the variance upper bound. What is surprising is that when using small $b/(M-1)$, increasing the number of the stepsizes used does not always reduce the $\ell_2$ error of the extrapolated iterates. This may be due to the proximity of stepsizes, leading to large coefficients $\{h_m\}$, increasing the variance and offsetting the reduction in bias. Hence, this suggests that for RR extrapolation to be effective, stepsizes should not be too close to each other, but should explore a range of values.

\begin{figure}[htbp]
\centering %
\begin{minipage}[c]{0.47\textwidth}%
\centering \includegraphics[width=1\textwidth]{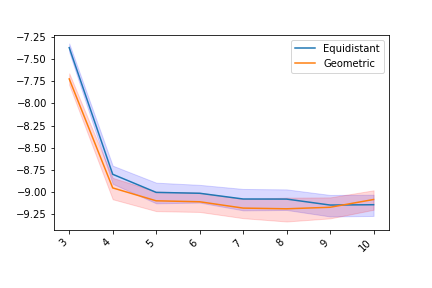}
\captionof{figure}{$\ell_2$ errors of averaged iterates with different order of RR extrapolation. The $Y$-axis is of logarithmic scale.} \label{fig:rr-eqd-geo} %
\end{minipage}\hspace{0.5cm}%
\begin{minipage}[c]{0.47\textwidth}%
 \centering \includegraphics[width=0.9\textwidth]{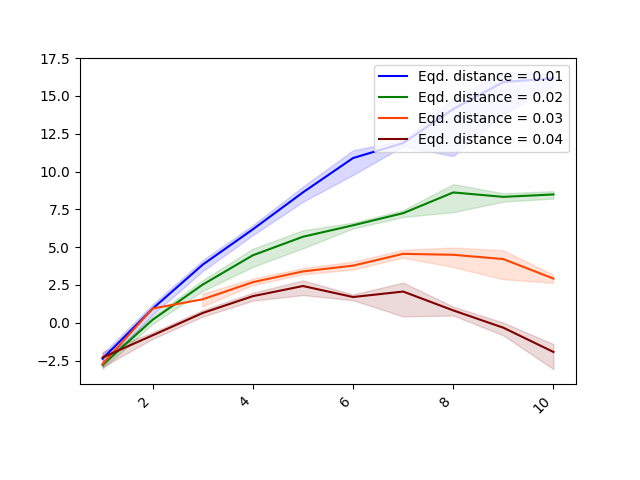}
\captionof{figure}{$\ell_2$ errors of averaged iterates with different order of RR extrapolation. The $Y$-axis is of logarithmic scale.
 } \label{fig:rr-eqd} %
\end{minipage}
\end{figure}

\section{Conclusion}

In this paper, we study and demonstrate the effectiveness of statistical inference with LSA iterates under Markovian data and constant stepsizes and RR extrapolation. We prove a CLT to theoretically justify the asymptotic normality of averaged iterates. We outline a statistical inference procedure that is combined with the RR extrapolation to construct CIs with better coverage probabilities of the target vector $\theta^\ast$. We conduct an extensive range of experiments to numerically examine the properties of our proposed inference framework, and compare the results against the classical approach with a diminishing stepsize sequence. Lastly, we also provide general guidance on how to select the stepsizes used in the RR extrapolation and present several interesting Markovian LSA scenarios where the iterates converge with zero asymptotic bias. 

Based on our work, an immediate next step is to provide more theoretical results to justify the validity of the confidence interval constructed with LSA iterates with Markovian data and constant stepsize, as the CI constructed in this fashion is non-consistent. Further future directions may also include: (a) develop an anytime variance estimator so that the inference procedure with constant stepsizes can be implemented in a fully online fashion; (b) given the simulation budget, i.e., fix the simulation trajectory length, how one should decide the order of RR extrapolation use to achieve decent inference performance.

\subsection*{Acknowledgements}
Y.\ Chen is supported in part by NSF CAREER Award CCF-2233152. Q.\ Xie is supported in part by NSF grant CNS-1955997.

\bibliographystyle{alpha}
\bibliography{cite}

\newpage{}

\appendix

\section{Properties of Markovian LSA}

Before proceeding to provide all proof of our theoretical results, we first recall several key properties and notations regarding the Markovian LSA with constant stepsize.

\paragraph*{Hurwitz Stability}
Under Assumption~\ref{assumption:hurwitz} that $\BarA$ is Hurwitz,  there exists a symmetric positive definite matrix $\Gamma$ such that $\BarA^\top\Gamma + \Gamma\BarA=-I$. We denote the largest and smallest eigenvalue of $\Gamma$ as $\gammax$ and $\gammin$ respectively. Additionally, we denote $\Gamma$-weighted norm as $\|v\|_\Gamma^2=\|v^\top\Gamma v\|_2$ and its induced operator norm as $\|M\|_\Gamma=\sup_{\|v\|_\Gamma=1}\|Mv\|_\Gamma$, where $M$ denotes some matrix.

\paragraph*{Geometric Mixing}

We first define the following mixing time for the underlying Markov chain $(x_t)_{t\geq0}$.

\begin{definition} For $\epsilon\in(0,1)$, the $\epsilon$-mixing
time of $(x_{t})_{t\geq0}$ with respect to $(A,b)$ is defined to
be the smallest number $\tau_{\epsilon}\geq1$ satisfying 
\begin{align}
\big\|\E[A(x_{t})\mid x_{0}=x]-\BarA\big\| & \leq\epsilon\cdot\Amax,\quad\forall x\in\cX,\,\forall t\geq\tau_{\epsilon},\label{eq:a-mix-time}\\
\big\|\E[b(x_{t})\mid x_{0}=x]-\bar{b}\big\| & \leq\epsilon\cdot\bmax,\quad\forall x\in\cX,\,\forall t\geq\tau_{\epsilon}.\label{eq:b-mix-time}
\end{align}
\end{definition}
Under Assumptions~\ref{assumption:uniform-ergodic} and~\ref{assumption:boundedness}, the Markov chain $(x_t)_{t\geq0}$ has mixing time 
\begin{equation*}
    \tau_{\epsilon}\leq K\log\frac{1}{\epsilon}\quad\forall\epsilon\in(0,1),
\end{equation*}
where the number $K\geq1$ is independent of $\epsilon$. Unless specified otherwise, we always choose $\epsilon=\alpha$. Owing to the order of $\tau_\alpha$, when $\alpha$ is sufficiently small, it is possible to upper bound the product $\alpha\tau_\alpha$ by some constant, i.e., $\alpha\tau_\alpha=\bigO(1)$. 

We would like to note that the following is another commonly used definition of mixing time.
\begin{definition} For $\epsilon\in(0,1)$, the $\epsilon$-mixing time of $(x_{k})_{k\geq0}$ is defined to
be the smallest number $t_\text{mix}(\epsilon)\geq1$ satisfying 
\begin{equation}
\sup_{x\in\cX}\big\|P^k(x,\cdot)-\pi(\cdot)\big\| \leq\epsilon,\quad\forall x\in\cX,\,\forall k\geq t_\text{mix}.\label{eq:chain-mix-time}
\end{equation}
\end{definition}
It is easy to see that $t_\text{mix}(\alpha)\geq\tau_\alpha$. Additionally, we define the mixing time $\widetilde\tau=\max(t_\text{mix}(1/4),\tau_\alpha)$, which will be used in the proof of CLT.

\paragraph*{Weak Convergence of LSA Iterates}

When the underlying data is Markovian, the iterates $(\theta_t)_{t\geq0}$ is no longer a standalone Markov chain. Hence, we expand the state space and consider the joint process $(x_t,\theta_t)_{t\geq0}$, which is a time-homogeneous Markov chain with constant stepsize. Under Assumptions~\ref{assumption:uniform-ergodic}--\ref{assumption:hurwitz}, it can be shown that the joint process converges weakly in Wasserstein distance.

To state this result, we first define the following metric and corresponding Wasserstein distance before stating the theorem.
\begin{equation}
    \Bar d\big((x,\theta),(x',\theta')\big):=\sqrt{d_{0}(x,x')+\|\theta-\theta'\|^{2}}.
\end{equation}
For a pair of distributions $\Bar{\mu}$ and $\Bar{\nu}$ in $\cP_{2}(\cX\times\R^{d}),$
we consider the Wasserstein-2 distance w.r.t.\ $\Bar d$: 
\begin{equation}
\begin{aligned}\Bar W_{2}(\Bar{\mu},\Bar{\nu}) & =\inf\bigg\{\Big(\E\big[\delta_{0}(x,x')+\|\theta-\theta'\|^{2}\big]\Big)^{1/2}:\;\law\big((x,\theta)\big)=\Bar{\mu},\;\law\big((x',\theta')\big)=\Bar{\nu}\bigg\}.\end{aligned}
\label{eq:w2-definition-extended}
\end{equation}

\begin{thm}[\cite{huo2022bias}] \label{thm:thm-converge} Suppose that Assumptions~\ref{assumption:uniform-ergodic}--~\ref{assumption:hurwitz} hold, there exists some threshold $\alpha_0\in(0,1)$ such that for stepsize $\alpha\in(0,\alpha_0)$, $(x_{t},\theta_{t})_{t\geq0}$ weakly converges to a unique limiting stationary distribution denoted by $\bar\mu$. Moreover, it holds that 
\begin{equation}
\label{eq:var-trace-order}
\tr(\var(\theta_{\infty}))=\bigO(\alpha\tau_\alpha) = \bigO(1).
\end{equation}
Additionally, For all $t\ge\tau_{\alpha}$, it holds that 
\begin{equation}
\label{eq:wasserstein-rate}
\bar{W}_{2}^{2}(\law(x_{t},\theta_{t}),\Bar{\mu})=\bigO\Big((1-c\alpha)^t\Big),
\end{equation}
where $c$ is some constant that is independent from stepsize $\alpha$ and iterate $t$.

\end{thm}

In fact, this theorem can be proven without the assumption of $x_0\sim\pi$. We present the proof of this generalization in Appendix~\ref{sec:proof-clt}. Such a generalization plays a crucial role in proving the desired CLT.

\paragraph*{Asymptotic Bias}
When the data is Markovian, the correlation among the data leads to an additional layer of nonlinear correlation among the LSA iterates, hence resulting in the presence of asymptotic bias \citep{huo2022bias} . 

\begin{thm} \label{thm:bias-characterization} 
Under Assumptions~\ref{assumption:uniform-ergodic}--~\ref{assumption:hurwitz} and some threshold $\alpha_1$ such that the stepsize $\alpha\in(0,\alpha_1)$, the asymptotic bias admits the following infinite series expansion 
\begin{equation}
\E[\theta_{\infty}^{(\alpha)}]-\theta^{\ast}=\sum_{i=1}^{\infty}\alpha^{i}B^{(i)},\label{eq:bias-expansion}
\end{equation}
where $B^{(i)}$ are vectors that are independent from the stepsize $\alpha$.
\end{thm}

Nonetheless, Markovian data does not always result in nonzero asymptotic bias. The work in \cite{huo2022bias} further provides the following sufficient condition for a zero asymptotic bias in a Markovian setting. 
\begin{cor} \label{cor:zero-bias-sufficient} Under the assumptions
of Theorem~\ref{thm:bias-characterization}, if 
\begin{equation}
\E\Big[A(x_{t})\theta^{\ast}+b(x_{t})\mid x_{t+1}=x\Big]=0,\quad\forall x\in\cX,\label{eq:zero-bias-sufficient}
\end{equation}
then $\E[\theta_{\infty}]-\theta^{\ast}=0$. \end{cor}
In this work, we take advantage of this sufficient condition and present scenarios that are commonly seen in practice that satisfy this sufficient condition.

\section{Delayed Proofs}
In this section, we provide detailed proof of our theoretical results.

\subsection{Proof of CLT (Theorem~\ref{thm:CLT})}
\label{sec:proof-clt}

\subsubsection{Removing Stationarity Assumption}
In \cite{huo2022bias}, the authors have proven that $(x_k,\theta_k)_{k\geq0}$ converges weakly to a unique limiting stationary distribution $(x_\infty,\theta_\infty)\sim\bar\mu$ under the assumption that $x_0\sim\pi$. In this paper, we are able to eliminate the stationarity start condition and prove that $(x_k,\theta_k)_{k\geq0}$ converges weakly to $\bar\mu$ for all $\mu_0$ on $(\cX,\cB(\cX))$ under some minor additional condition.

To remove the assumption that $(x_k,\theta_k)_{k\geq0}$, we prove that for a chain $(x_k,\theta_k)_{k\geq0}$ that starts from initial distribution $\mu_0$ also weakly converges to $\bar\mu$. As such, this implies that $x_0$ does not need to be initialized from $\pi$ in order to converge to $\bar\mu$.

We next state our new improved assumption and convergence theorem.
\begin{assumption}
\label{assumption:uniform-ergodic-new}
    $(x_t)_{t\geq0}$ is a uniformly ergodic Markov chain with transition kernel $P$ and a unique stationary distribution $\pi$.
\end{assumption}

\begin{thm}
\label{thm:new-convergence}
    Suppose that Assumption~\ref{assumption:uniform-ergodic-new} and Assumption~\ref{assumption:boundedness} hold, and that the stepsize $\alpha$ satisfies
    \begin{equation}
        \alpha\widetilde\tau < \min\Big(\frac{0.05}{95\gammax}, \alpha_\infty^{(M)}, \frac{1}{16\sqrt{\gammax}}\Big(\frac{1}{\sqrt{\gammin}}+1\Big)^{-1}\Big),
    \end{equation}
    where
    \begin{align}
        \alpha_\infty^{(M)}&=\min\Big(\alpha_\infty,\sqrt{\frac{\gammin}{\gammax}}, \frac{\gammin}{12e\gammax^2}\Big)\times\left\lceil\frac{16}{\sqrt{\gammax\gammin}}\,\right\rceil^{-1}\\
        \alpha_\infty&=\min\Big(\frac{1}{2\gammax}\|\BarA\|_\Gamma^2, \gammax\Big).
    \end{align}
    Under all initial distributions of $(x_0,\theta_0)\sim{\mu_0}$ with $\E_{\mu_0}[\|\theta_0\|^8]<\infty$, the sequence of random variables $(x_k,\theta_k)_{k\geq0}$ converges in $\bar{W}_2$ to a unique limit $(x_\infty,\theta_\infty)\sim\bar\mu$. 
    
\end{thm}

\begin{proof}
    For the LSA update
    \begin{equation}
\theta_{k+1}^{(\alpha)}=\theta_{k}^{(\alpha)}+\alpha\left(A(x_{k})\theta_{k}^{(\alpha)}+b(x_{k})\right),\quad k=0,1,\ldots\label{eq:update-rule}
\end{equation}
we consider two processes $(x_k,\theta_k)_{k\geq0}$ and $(x'_k,\theta'_k)_{k\geq0}$ with the joint initialization $\xi\sim\gamma\Big((x_0,\theta_0),(x'_0,\theta'_0)\Big)$, where $(x'_0,\theta'_0)\sim\bar\mu$ and $(x_0,\theta_0)\sim\mu_0$, and the coupling is chosen as the optimal coupling that achieves the Wasserstein distance, i.e.,
\begin{equation}
    \label{eq:initial-coupling}
\Bar W_{2}^{2}\left(\law\big(x_{0},\theta_{0}\big),\law\big(x_{0}',\theta_{0}'\big)\right)  =\E_\xi\left[d_{0}\big(x_{0},x_{0}'\big)+\big\|\theta_{0}-\theta_{0}'\big\|^{2}\right]\quad\text{and}.
\end{equation}
Additionally, the two processes are coupled in the following fashion,
\begin{equation}
x_{k+1}  =x_{k+1}'\;\;\text{ if }x_{k}=x_{k}',\quad\forall k\ge0.
\end{equation}

Recalling the definition of the discrete metric
$d_{0}(x'_k,x_k):=\mathbf{1}\left\{ x'_k\neq x_k\right\} $, we
have the identities 
\begin{align*}
A(x_k) & =A(x'_k)+d_k(x'_k,x_k)\cdot\big(A(x_k)-A(x'_k)\big)\quad\text{and}\\
b(x_k) & =b(x'_k)+d_k(x'_k,x_k)\cdot\big(b(x_k)-b(x'_k)\big).
\end{align*}
The update rule~\eqref{eq:update-rule} together with the above identities
implies that 
\begin{equation}
\label{eq:theta-diff-iter}
\begin{aligned}
    \theta_{k+1}-\theta'_{k+1}= & \theta_k+\alpha\big(A(x_k)\theta_k+b(x_k)\big)-\theta'_k-\alpha\big(A(x'_k)\theta'_k+b(x'_k)\big)\\
= & \big(I+\alpha A(x'_k)\big)\cdot\big(\theta_k-\theta'_k\big)+\alpha d_0(x'_k,x_k)\cdot\big[\big(A(x_k)-A(x'_k)\big)\theta_k+b(x_k)-b(x'_k)\big].
\end{aligned}
\end{equation}

Next, we denote the following shorthands,
\begin{align}
    w_k&:=\theta_k-\theta'_k\\
    A_k&:=A(x_k')\\
    g_k&:=d_0(x_k',x_k)\cdot\big[\big(A(x_k)-A(x'_k)\big)\theta_k+b(x_k)-b(x'_k)\big].
\end{align}
Hence, we rewrite~\eqref{eq:theta-diff-iter} using the above defined shorthands and have
\begin{equation}
    w_{k+1}=(I+\alpha A_k)w_k+\alpha g_k.
\end{equation}
Recursively, we have
\begin{align}
    w_k&=\prod_{l=0}^{k-1}(I+\alpha A_l)w_0+\alpha\sum_{l=0}^{k-1}\Big[\prod_{m=l}^{k-2}(I+\alpha A_m)\Big]g_l.
\end{align}
We further denote
\begin{equation}
    B_{1:k}^{(\alpha)}:=\prod_{l=0}^{k}(I+\alpha A_l),
\end{equation}
and we obtain
\begin{equation}
    w_k=B_{0:k-1}w_0 + \alpha\sum_{l=0}^{k-1}B_{l:k-2}g_l.
\end{equation}

We first note that
\begin{align}
    \|w_k\|^2&=\Big\|B_{0:k-1}w_0 + \alpha\sum_{l=0}^{k-1}B_{l:k-2}g_l\Big\|^2\\
    &\leq 2\Big\|B_{0:k-1}w_0\Big\|^2 + 2\alpha^2\Big\|\sum_{l=0}^{k-1}B_{l:k-2}g_l\Big\|^2\\
    &\leq 2\|B_{0:k-1}w_0\|^2 + 2\alpha^2k\sum_{l=0}^{k-1}\|B_{l:k-2}g_l\|^2.
\end{align}

We next bound the terms on the RHS respectively.
\begin{align}
    \E_\xi\Big[\|B_{0:k-1}w_0\|^2\Big]&\leq\Big(\E_\xi\Big[\|B_{1:k-1}\|^4\Big]\E_\xi\Big[\|(I+\alpha A_0)w_0\|^4\Big]\Big)^{1/2}\\
    &=\Big(\E_\xi\Big[\|B_{1:k-1}\|^4\Big]\Big)^{1/2}\Big((1+\alpha)^4\E_\xi\Big[\|\theta_0-\theta'_0\|^4\Big]\Big)^{1/2}\label{eq:cauchy-step1}\\
    &\lesssim  \rho_1^{\alpha k},
\end{align}
where $\rho_1\in(0,1)$ and ``$\lesssim$" hides constant coefficients that are independent from $k$. 
To prove the last inequality, we make use of \cite[Proposition 7]{durmus22-LSA} to bound the first term of \eqref{eq:cauchy-step1}, and the following relationship to bound the second term,
\begin{equation}
    \E_\xi\Big[\|\theta_0-\theta_0'\|^4\Big]\leq 8\Big(\E_\xi[\|\theta_0\|^4]+\E_\xi[\|\theta'_0\|^4]\Big)<\infty,
\end{equation}
for $\E_\xi[\|\theta_0'\|^4]<\infty$ has been proven in \cite{huo2022bias} and $\E_\xi[\|\theta_0\|^4]<\infty$ by assumption.

For the remaining terms, we first note that
\begin{align}
    \E_\xi\Big[\|B_{l:k-2}g_l\|^2\Big]&\leq\E_\xi\Big[\|B_{l+1:k-2}\|^2\|(I+\alpha A_l)g_l\|^2\Big]\\
    &=\E_\xi\Big[\E_{\xi_l}\Big[\|B_{l+1:k-2}\|^2\|(I+\alpha A_l)g_l\|^2\Big]\Big]\\
    &\leq \E_\xi\Big[\Big(\E_{\xi_l}\Big[\|B_{l+1:k-2}\|^4\Big]\E_{\xi_l}\Big[\|(I+\alpha A_l)g_l\|^4\Big]\Big)^{1/2}\Big]\\
    &\leq \Big(\E_\xi\Big[\E_{\xi_l}\Big[\|B_{l+1:k-2}\|^4\Big]\E_{\xi_l}\Big[\|(I+\alpha A_l)g_l\|^4\Big]\Big]\Big)^{1/2}\\
    &\lesssim \rho_1^{k-l}\Big(\E_\xi\Big[\|g_l\|^4\Big]\Big)^{1/2},
\end{align}
where $\xi_l$ denotes the distribution of the joint distribution of the two processes at time $l$ under the initial distribution $\xi$, i.e. $\xi_l\sim\Big((x_l,\theta_l),(x_l',\theta_l')\Big)\mid\xi$.

Next, we proceed to bound the second term of $g_l$. Applying Cauchy-Schwarz again, we have
\begin{align}
    \E_{\xi}\Big[\|g_l\|^4\Big]&=\E_{\xi}\Big[\|d_0(x_l',x_l)\cdot\big[\big(A(x_l)-A(x'_l)\big)\theta_l+b(x_l)-b(x'_l)\big]\|^4\Big]\\
    &\leq \Big(\E_\xi\Big[d_0(x_l',x_l)\Big]\E_\xi\Big[\|\big(A(x_l)-A(x'_l)\big)\theta_l+b(x_l)-b(x'_l)\|^8\Big]\Big)^{1/2}.
\end{align}
By uniform ergodicity of $(x_k)_{k\geq0}$ in Assumption~\ref{assumption:uniform-ergodic}, we first have that
\begin{equation}
    \E_\xi\Big[d_0(x'_l,x_l)\Big]\lesssim\rho_2^l,\quad\rho_2\in(0,1).
\end{equation}
Next, we observe that
\begin{equation}
    \E_\xi\Big[\|\big(A(x_l)-A(x'_l)\big)\theta_l+b(x_l)-b(x'_l)\|^8\Big]\leq \E_\xi\Big[\Big(2(\Amax\theta_l+\bmax)\Big)^8\Big]\leq C_3,
\end{equation}
where the last inequality holds by \cite[Theorem 9]{srikant-ying19-finite-LSA}.
Hence, we obtain
\begin{equation}
    \E_\xi\Big[\|B_{l:k-2}g_l\|^2\Big]\lesssim\rho_1^{k-l-1}\rho_2^l\lesssim\rho_3^{k-1},
\end{equation}
with $\rho_3=\max(\rho_1,\rho_2)$.

Combining all the pieces, we obtain that
\begin{align}
     \E_\xi\Big[\|w_k\|^2\Big]&\lesssim\rho_1^{k} + k^2\rho_3^k\rightarrow0.
\end{align}
Therefore, it is easy to see that
\begin{equation}
    \Bar{W}_2^2\Big((x_k,\theta_k),(x'_k,\theta_k')\Big)\leq \E\Big[\|w_k\|^2 +d_0(x_k,x_k')\Big]\to0.
\end{equation}
Hence, $\bar\mu$ is the unique limiting stationary distribution for $(x_k,\theta_k)_{k\geq0}$, regardless of the initial distribution of $(x_k)_{k\geq0}$.

As such, this has completed the proof of this theorem.
\end{proof}

Even under the relaxed initial distribution, the convergence rate can still be proven in a similar fashion as in \cite{huo2022bias} and obtain for $k\geq\widetilde\tau$,
\begin{equation}
\bar{W}_{2}^{2}(\law(x_{k},\theta_{k}),\Bar{\mu})\leq20\,\frac{\gammax}{\gammin}\Big(\E[\|\theta_{0}-\E[\theta_{\infty}]\|^{2}]+\tr(\var(\theta_{\infty}))\Big)\cdot\left(1-\frac{0.9\alpha}{\gammax}\right)^{k}.\label{eq:w2-thetak-to-mu}
\end{equation}

\subsubsection{Proving CLT (Theorem~\ref{thm:CLT}) using Theorem~\ref{thm:new-convergence}.}

Now with Theorem~\ref{thm:new-convergence}, we are ready to prove the CLT for averaged iterates of Markovian LSA with constant stepsize.

\begin{proof}
Consider function $h:\R^d\times\cX\to\R^d$ defined as 
\begin{equation}
    h(x, \theta)=\theta-\E[\theta_\infty].
\end{equation}
To prove that the CLT for $h$, we need to verify the Maxwell-Woodroofe condition \citep{MaxwellWoodroofe-2000}, i.e.,
\begin{equation}
    \sum_{n=1}^\infty n^{-3/2}\Big\|\sum_{t=0}^{n-1}Q^th\Big\|_{L^2(\bar\mu)}<\infty,
\end{equation}
where $Q$ denotes the transition kernel of the joint Markov chain. It is easy to see that, if we can show the following order bound
\begin{equation}
\label{eq:order-maxwell}
    \Big\|\sum_{t=0}^{n-1}Q^th\Big\|_{L^2(\bar\mu)} = \bigO(n^r)
\end{equation}
with $r\in[0,1/2)$, then the Maxwell-Woodroofe condition is verified for
\begin{equation}
 \sum_{n=1}^\infty n^{-3/2}\Big\|\sum_{t=0}^{n-1}Q^th\Big\|_{L^2(\bar\mu)} =  \sum_{n=1}^\infty n^{-3/2}\bigO(n^r)<\infty.
\end{equation}

We now prove the desired order in \eqref{eq:order-maxwell}. For sufficiently large $n$ such that $n\geq\widetilde\tau$, we have
\begin{align}
    \Big\|\sum_{t=0}^{n-1}Q^th\Big\|_{L^2(\bar\mu)}=\E_{\bar\mu}\Big\|\sum_{t=0}^{n-1}Q^th\Big\|_2&\leq \sum_{t=0}^{n-1}\E_{\bar\mu}\|Q^th\|_2\\
    &=\underbrace{\sum_{t=0}^{\tau_\alpha-1}\E_{\bar\mu}\|Q^th\|_2}_{T_1} + \underbrace{\sum_{t=\tau_\alpha}^{n-1}\E_{\bar\mu}\|Q^th\|_2}_{T_2}.
\end{align}

We now show that both terms $T_1$ and $T_2$ are of order $\bigO(1)$ with respect to the parameter $n$.

For $T_1$, since $Q$ is a transition kernel, so its $\|Q\|_{L^2(\mu)}$ operator norm equals to $1$. Hence, $T_1$ can be upper bounded as 
\begin{align}
    T_1\leq \tau_\alpha\E_{\bar\mu}[\|h(\theta,x)\|_2^2] = \tau_\alpha\text{Tr}(\text{var}(\theta_\infty))<C_1,
\end{align}
where the last inequality follows from $\text{Tr}(\text{var}(\theta_\infty))<C$ established in \eqref{eq:var-trace-order}.

Before proceeding to analyze the summation in $T_2$, we first recall \eqref{eq:wasserstein-rate}, that for $t\geq\tau_\alpha$,
\begin{equation}
    \bar{W}^2_{2}(\law(x_t,\theta_t),\bar\mu)=\bigO((1-c\alpha)^{t}),
\end{equation}
which holds for any $(x,\theta)\in\cX\times\R^d$.
Hence, by the property of Wasserstein distance \citep{Villani08-ot_book}, there always exists a coupling that attains the optimality, i.e.,
\begin{equation}
\label{eq:opt-coupling}
    \E_{\Gamma((x_t,\theta_t),\bar\mu)}\Big[\|\theta_k-\theta'\|_2^2+\delta_0(x_k\neq x')\Big]=\bigO((1-c\alpha)^t).
\end{equation}
Making use of this relationship, we can therefore bound $T_2$,
\begin{equation}
    T_2=\sum_{t=\tau_\alpha}^{n-1}\E_{\bar\mu}\|Q^th\|_2
    \leq\sum_{t=\tau_\alpha}^\infty\E_{\bar\mu}\|Q^th\|_2
    =\bigO\Big(\frac{1}{1-\sqrt{1-c\alpha}}\Big) = \bigO(1),
\end{equation}
where the last $\bigO(\cdot)$ is asymptotic in $n$.

Combining the analysis of $T_1$ and $T_2$, we have shown the desired order in \eqref{eq:order-maxwell}. Therefore, the CLT for averaged LSA iterates with constant stepsize and Markovian data holds true for the Maxwell-Woodroofe condition has been verified.

This has completed the proof of Theorem~\ref{thm:CLT} for the CLT.
\end{proof}

\subsection{Proof of RR Variance Bounds (Propositions~\ref{prop:rr-geo}--\ref{prop:rr-eqd})}
\label{sec:proof-rr-var}

Recall that the RR extrapolated iterates is simply a linear combination of LSA iterates under different constant stepsizes, as shown again in the following equation,
\begin{equation}
    \widetilde\theta_\infty=\sum_{k=1}^Mh_k\theta_\infty^{(\alpha_k)},
\end{equation}
The coefficients $h_k$ in RR extrapolation are selected as the solution to the following system of linear system. Making use of the properties of Vandermonde Matrix, we are able to solve for $h_k$ explicitly in values of stepsizes $\alpha_m$ for $m=1,\ldots, M$, as shown in \eqref{eq:h-formula}.
\begin{align}
    &\begin{bmatrix}
        1 & 1 & \cdots & 1\\
        \alpha_1 & \alpha_2 & \cdots & \alpha_M\\
        \vdots\\
        \alpha_1^{M-1} & \alpha_2^{M-1} & \cdots & \alpha_M^{M-1}
    \end{bmatrix}\begin{bmatrix}
        h_1\\ h_2\\\vdots\\ h_M
    \end{bmatrix}=\begin{bmatrix}
        1 \\ 0 \\\vdots\\0
    \end{bmatrix}\\
    &h_m=\prod_{l=1,l\neq m}^M\frac{\alpha_l}{\alpha_l-\alpha_m}\label{eq:h-formula}
\end{align}

Now that we have an explicit formula for the coefficient $h_k$, we can use it to analyze the asymptotic variance of the RR extrapolated iterates.

\subsubsection{Proof of Variance Bound under Geometric Decay Schedule (Proposition~\ref{prop:rr-geo})}
\label{sec:proof-var-geo}
In this section, we prove the upper bound to the asymptotic variance of the RR extrapolated iterate under geometric decay schedule, that $\alpha_1<1$ and $\alpha_k=c^{-(m-1)}\alpha_1$ for $m=1,\ldots, M$ for $c\geq 2$.

\begin{proof}
    In the geometric decay schedule, the stepsizes are selected in the following fashion,
    \begin{equation}
        \alpha_1<1,\quad c\geq 2,\quad\text{and}\quad \alpha_m=\alpha_1/c^{(m-1)}\quad m=1,2,\ldots, M.
    \end{equation}
    Substituting $\alpha_m$ into \eqref{eq:h-formula}, we obtain the following expression of $h_m$
    \begin{equation}
        |h_m|=\Big(\prod_{l=1}^{m-1}\frac{c^l}{c^l-1}\Big)\Big(\prod_{l=1}^{M-m}\frac{1}{c^l-1}\Big).
    \end{equation}
    It is easy to see that when $c\geq 2$, we have $c^l-1\geq1$ for $l\geq1$. Hence, $|h_m|\leq|h_M|=\Big(\prod_{l=1}^{M-1}\frac{c^l}{c^l-1}\Big)$.

    Next, we notice that when $c\geq2$, 
    \begin{equation}
        \prod_{l=1}^{M-1}\frac{c^l}{c^l-1}\leq\prod_{l=1}^\infty\frac{c^l}{c^l-1} = \Big(\prod_{l=1}^\infty (1-c^{-l})\Big)^{-1}\leq h_{\max}(c).
    \end{equation}
    
    We obtain the following upper bound to the infinite product and delay the proof to the end of this section.
    \begin{claim}
    \label{clm:h-max}
        $h_{\max}(c)=\exp\Big(\frac{2}{c-1}\Big)$.
    \end{claim}

    Before proceeding to give the upper bound, we recall a key property of $\var(\theta_\infty^{(\alpha_m)})$ established in \cite{huo2022bias}.
    \begin{equation}
        \label{eq:var-bound}
        \var(\theta_\infty^{(\alpha_m)})=\bigO(\alpha_m\tau_{\alpha_m}) = \bigO(1).
    \end{equation}

    Making use of the value of $h_{\max}(c)$, we now are above to establish the upper bound to the variance of the RR extrapolated iterate.
    \begin{align}
        \var(\widetilde\theta_\infty)&\leq\Big(\sum_{k=1}^M|h_k|\sqrt{\var(\theta_\infty^{(\alpha_k)})}\,\,\Big)^2
        \lesssim\Big(\sum_{k=1}^M|h_k|\sqrt{\alpha_k\tau_{\alpha_k}}\Big)^2\\
        &\lesssim \exp\Big(\frac{4}{c-1}\Big)\cdot\Big(\sum_{k=1}^M(c^{-(k-1)}\alpha_1)^{1/2}(\log\frac{1}{c^{-(k-1)}\alpha_1})^{1/2}\Big)^2\\
        &\leq\alpha_1\cdot \exp\Big(\frac{4}{c-1}\Big)\cdot \Big(\sum_{k=0}^\infty c^{-k/2}\Big(k\log c + \log(1/\alpha_1)\Big)^{1/2}\Big)^2\\
        &\lesssim\exp\Big(\frac{4}{c-1}\Big)\cdot
        \log^2c\cdot\Big(\sum_{k=0}^\infty c^{-k/2}\cdot k \Big)^2\\
        &\lesssim \exp\Big(\frac{4}{c-1}\Big)\cdot c\cdot\Big(1-c^{-1/2}\Big)^{-4}\\
        &\leq \exp\Big(16\, c^{-1/2}\Big)\cdot c.
    \end{align}
    
    As such, we have proven the desired upper bound and completed the proof of Proposition~\ref{prop:rr-geo}.
\end{proof}

\begin{proof}[Proof of Claim~\ref{clm:h-max}]
    To obtain a valid $h_{\max}(c)$, we consider lower bound the infinite product $\prod_{l=1}^\infty (1-c^{-l})$. We make use of exponential properties, and have
    \begin{equation}
        \prod_{l=1}^\infty (1-c^{-l})
        =\prod_{l=1}^\infty \exp\Big(\log(1-c^{-l})\Big) 
        =\exp\Big(\sum_{l=1}^\infty \log(1-c^{-l})\Big).
    \end{equation}
    
    We recall the following inequality, that for $|z|<1/2$, 
    \begin{equation}
        |\log(1+z)|\leq 2|z|.
    \end{equation}
    Substituting $z$ with $c^{-1}$, we obtain
    \begin{equation}
        |\log(1-c^{-l})|\leq 2|c^{-l}|,\quad\forall\, l\geq1.
    \end{equation}
    Hence, a valid bound to the infinite product is
    \begin{equation}
        \exp\Big(-2\sum_{l=1}^\infty c^{-l}\Big)
        \leq \exp\Big(\sum_{l=1}^\infty \log(1-c^{-l})\Big)\\
        \leq \exp\Big(2\sum_{l=1}^\infty c^{-l}\Big).
    \end{equation}
    As such, we set
    \begin{equation}
        h_{\max}(c)=\exp\Big(2\sum_{l=1}^\infty c^{-l}\Big)=\exp\Big(\frac{2}{c-1}\Big),
    \end{equation}
    thus proving Claim~\ref{clm:h-max}.
\end{proof}

\subsubsection{Proof of Variance Bound under Equidistant Decay Schedule (Proposition~\ref{prop:rr-eqd})}
\label{sec:proof-var-eqd}
In this section, we prove the upper bound to the asymptotic variance of the RR extrapolated iterate under equidistant decay schedule, that $a+b<1$ and $\alpha_k=(a+b)-\frac{b(m-1)}{M-1}$ for $m=1,\ldots, M$.
\begin{proof}
    Under the general equidistant decay schedule, we have the stpesize $\alpha_m$ selected in the following manner, assuming $a + b <1$,
    \begin{equation}
        \alpha_m=(a + b) -b\frac{m-1}{M-1}, \quad m = 1,\ldots, M. 
    \end{equation}
    Substituting $\alpha_m$ into \eqref{eq:h-formula}, we obtain the following expression of $h_m$,
    \begin{equation}
        |h_m|=\binom{M}{m}\cdot\frac{bm}{a(M-1)+b(M-m)}\cdot\Big(\prod_{l=1}^M\frac{a(M-1)+b(l-1)}{bl}\Big).
    \end{equation}
    Now that we have the values of $|h_m|$, we are able to find the upper bound for the variance of the RR extrapolated iterate.
    \begin{align}
        \var(\widetilde\theta_\infty)&\leq\Big(\sum_{m=1}^M|h_m|\sqrt{\var(\theta_\infty^{(\alpha_m)})}\Big)^2
        \lesssim\Big(\sum_{mk=1}^M|h_m|\sqrt{\alpha_m\tau_{\alpha_m}}\Big)^2\\
        &\leq\Big(\sum_{m=1}^M\binom{M}{m}\cdot\frac{bm}{a(M-1)+b(M-m)}\cdot\Big(\prod_{l=1}^M\frac{a(M-1)+b(l-1)}{bl}\Big)\cdot\sqrt{\alpha_m\tau_{\alpha_m}}\Big)^2\\
        &\lesssim \Big(\frac{bM}{a(M-1)}\cdot\Big(\prod_{l=1}^M\frac{a(M-1)+b(l-1)}{bl}\Big)\cdot\sum_{m=1}^M\binom{M}{m}\Big)^2\\
        &\lesssim \frac{b}{a}\cdot\Big(\Big(\prod_{l=1}^M\frac{(a+b)M}{bl}\Big)\cdot2^M\Big)^2 = \bigO((2M/b)^{2M}),
    \end{align}
    where the $\bigO(\cdot)$ notation hides constant coefficient that are independent from $M$.

    Hence, we have completed the proof of Proposition~\ref{prop:rr-eqd}.
\end{proof}

\subsection{Proof of Zero Bias When Independent Multiplicative Noise Only (Proposition~\ref{prop:lsa-indep-mul})}
\label{sec:proof-mkv-add}
In this section, we prove that when we only have independent multiplicative noise in LSA, the LSA iterates converge without any asymptotic bias.
\begin{proof}
    We verify this condition by directly examining the iteration at stationarity.
    \begin{equation}
        \E[\theta_{\infty + 1}]=\E[\theta_\infty] + \alpha \E[(\BarA+\xi)\theta_\infty+b(x_\infty)].
    \end{equation}
    Hence, we have
    \begin{equation}
        \E[(\BarA+\xi)\theta_\infty+b(x_\infty)] = 0.
    \end{equation}
    Since $\xi$ is mean-zero and independent, we can simplify the above equation to
    \begin{equation}
        \Big(\BarA + \cancel{\E[\xi]}\Big)\E[\theta_\infty] + \Bar{b} = 0,
    \end{equation}
    and we solve for 
    \begin{align}
        \E[\theta_\infty] = -\BarA^{-1}\Bar{b} = \theta^\ast.
    \end{align}

    As such, we have proven that there is no bias when we only have independent multiplicative noise and completed the proof of Proposition~\ref{prop:lsa-indep-mul}.
\end{proof}

\subsection{Proof of Zero Bias in Realizable Linear-TD (Proposition~\ref{prop:td-realize})}
\label{sec:proof-realize}

\begin{proof}
    Under the realizable assumption, we have a vector $\theta^\ast\in\R^d$ such that
    \begin{align}
        \phi(s)^\top\theta^\ast &= r(s)+\gamma \int_{s'\in\cS}P(s,\dd s')\phi(s')^\top\theta^\ast\\
        &=r(s)=\gamma \E[\phi(s^\text{next})^\top\theta^\ast\mid s].
    \end{align}

    To show that there is zero asymptotic bias in the semi-simulator set-up, we only need to verify the conditional expectation to always equal zero, which has been shown in Corollary~\ref{cor:zero-bias-sufficient} that this conditional expectation being zero is a sufficient condition to zero asymptotic bias.

    Recall that $x_k=(s_k, s_k^\text{next})$, $x_{k+1}=(s_{k+1}, s_{k+1}^\text{next})$ in the semi-simulator set-up. Hence, we have
    \begin{align}
        \E\Big[A(x_k)\theta^\ast + b(x_k)\mid x_{k+1}\Big]
        &=\E\Big[\phi(s_k)\Big(\gamma\phi(s_k^\text{next})-\phi(s_k)\Big)^\top\theta^\ast + r(s_k)\phi(s_k)\mid x_{k+1}\Big]\\
        &=\E\Big[\phi(s_k)\Big((\gamma\phi(s_k^\text{next})-\phi(s_k))^\top\theta^\ast + r(s_k)\Big)\mid x_{k+1}\Big]\\
        &=\E\Big[\phi(s_k)\underbrace{\E\Big[(\gamma\phi(s_k^\text{next})-\phi(s_k))^\top\theta^\ast + r(s_k)\mid s_k, \cancel{x_{k+1}}\Big]}_{=0\text{ by realizability assumption}}\mid x_{k+1}\Big]\\
        &=0.
    \end{align}

    Hence, we have shown that when the linear-TD is realizable, the iterates converge without any asymptotic bias in the semi-simulator setting. Therefore, we have completed the proof of Proposition~\ref{prop:td-realize}.

\end{proof}

\section{Supplementary for Numerical Experiments}
\label{sec:exp-details}

In this section, we present details for our numerical experiments. All experiments are run on Intel Xeon Gold 6154 18-core CPUs with 566Gb RAM. 

\subsection{Setup for LSA Experiments}
\label{sec:lsa-exp-details}

For the experiments on LSA, we generate the transition probability
matrix $P$ and functions $A$ and $b$ randomly on an 10-state finite
state space as follows.

We first illustrate the steps we take to generate the transition matrix
$P$. For a given $n\,(=|\cX|)$, we start with a random matrix $M^{(P)}\in[0,1]^{n\times n}$
with entries $m_{ij}^{(P)}\overset{\text{i.i.d.\ }}{\sim}U[0,1]$,
and normalize it to obtain a stochastic matrix $\hat{M}^{(P)}=\left(\hat{m}_{ij}^{(P)}\right)$
with $\hat{m}_{ij}^{(P)}=\frac{m_{ij}^{(P)}}{\sum_{k=1}^{n}m_{ik}^{(P)}}$.
We then examine the period and reducibility of the stochastic matrix
$\hat{M}^{(P)}$ to ensure that it is aperiodic and irreducible, which
gives a uniformly ergodic Markov chain as required in Assumption~\ref{assumption:uniform-ergodic}.
If $\hat{M}^{(P)}$ is not aperiodic or irreducible, we then repeat
the above procedure until we obtain one, and set $P:=\hat{M}^{(P)}$.
Now with $P$ generated, we compute the stationary distribution $\pi$.

Next, we proceed to generate $A(x)\in\R^{d\times d}$ for $x\in\cX$ and $d=5$. As we also need
$\BarA=\E_{\pi}[A(x)]$ Hurwitz as required in Assumption~\ref{assumption:hurwitz},
we start with generating the Hurwitz matrix $\BarA$ and then add
noise to obtain the respective $A(x)$. We first generate a random
matrix $M^{(A)}\in\R^{d\times d}$ with $m_{ij}^{(A)}\overset{\text{i.i.d.\ }}{\sim}\mathcal{N}(0,1)$,
and examine the eigenvalues $\lambda_{i}(M^{(A)})$, as Hurwitz matrix
has eigenvalues all with strictly negative real parts. If $\text{Re}(\lambda_{i}(M^{(A)}))<0$
for all $i=1,\ldots,d$, then $M^{(A)}$ is Hurwitz and we set it
as $\BarA:=M^{(A)}$. Otherwise, we adjust $M^{(A)}$ to obtain a
Hurwitz matrix, $\BarA:=M^{(A)}-2\max(\text{Re}(\lambda_{i}(M^{(A)})))\cdot I_{d}$.
With $\BarA$ generated, we add a noise matrix $E(x)\in[-1,1]^{d\times d}$
to $\BarA$ to obtain $A(x)$, i.e., $A(x)=\BarA+E(x)$. As $\E_{\pi}[E(x)]=0$,
we only generate $E(x)$ with $e(x)_{ij}\overset{\text{i.i.d.\ }}{\sim}U[-1,1]$
for $x=1,\ldots,n-1$, and set $A(n)=\BarA-\sum_{x=1}^{n-1}\pi_{x}E(x)$.

Lastly, we generate $b(x)\in\R^{d}$ with $b(x)_{i}\overset{\text{i.i.d.\ }}{\sim}[-1,1]$
and obtain $\bar{b}=\sum_{x}\pi_{x}b(x)$ and $\bmax=\max_{x}\|b(x)\|$.

\subsection{Inference with LSA iterates with diminishing stepsize}
\label{sec:inf-proce-dimin}

In our paper, we compare our RR extrapolated constant stepsize inference procedure with the classical diminishing stepsize sequence. In our experiments, we follow the batch-mean inference procedure with diminishing stepsize outlined in \cite{ChenXi2020}.
The procedure with diminishing stepsize sequence itself is nearly the same as what we delineated for constant stepsizes in Section~\ref{sec:inference-proc}. Given a single trajectory of the Markov chain $(x_k)_{k\geq0}$, we run LSA with constant stepsizes and obtain a corresponding trajectory of $(\theta_k)_{k\geq0}$ iterates. The $(\theta_k)_{k\geq0}$ iterates are divided into batches for batch-mean estimation, with the first batch considered as a burn-in period and discarded. 

The major difference between the procedure of diminishing stepsize and that of constant stepsize is that the batch length used in the diminishing stepsize approach is not fixed but rather increases exponentially in length as the stepsize decreases, so as to overcome the increasing correlation among consecutive iterates. In our experiments, we follow the batch size selection guidance outlined in \cite{ChenXi2020}.
Given the total number of iterations $T$ and the total number of batches $K$, the iterates are divided into batches with index $e_k$ marking the end of batch $k$. The end indices $e_k$ are computed as
\begin{equation}
    e_k=\Big((k+1)r\Big)^{\frac{1}{1-\alpha}},\quad k=0,\ldots, K,
\end{equation}
with $r$ being a decorrelation strength factor computed as $r=\frac{T^{1-\alpha}}{K+1}$.
Since LSA iterations with diminishing stepsize converge almost surely, we do not additionally apply RR extrapolation to the inference procedure.

\subsection{Baseline Comparison for i.i.d.\ LSA}
\label{sec:baseline-iid}
Both of the two inference procedures, either with constant stepsize or diminishing stepsize, are originally proposed for SA with i.i.d.\ data \citep{li2017-constantine, ChenXi2020, XieZhang_SAInference_pku}. However, there has not been any cross-study across these two regimes, so we first examine the performance of these two different inference regimes under i.i.d.\ LSA to serve as a baseline for comparison with inference with Markovian data. 

We examine 100 different LSA problems of the same dimension with $|\cX|=10$ and $d=5$. For each LSA problem, the parameters $(P,A,b)$ are generated randomly. For each set of $(P, A, b)$ LSA problem, we run 100 different realized trajectories $(x_k)_{k\ge0}$ of length $10^5$ and run LSA iterates with the above-described stepsize regimes and perform statistical inference respectively. 
As the asymptotic bias is zero for LSA with constant stepsizes and i.i.d.\ data, we do not use RR extrapolation in this set of experiments when inference with constant stepsizes. 
The number of batches for both constant and diminishing stepsize used here is carefully set at $50$, which is at the recommended level for diminishing stepsize discussed in \cite{ChenXi2020}. 

We record the percentiles for $\ell_2$ error, CI width, and coverage percentage across the 100 different setups, with key statistics presented in Table~\ref{tab:quantile table iid}. Note that the confidence intervals and coverage probabilities data recorded in all tables focus on the 1st coordinate only. 

The diminishing stepsize sequence starting from $0.02$ converges too slowly, and hence the inference performance is consistently the worst across all three aspects out of the four stepsize sequences. For the remaining three stepsize sequences, the coverage properties are generally comparable, but the constant stepsizes always outperform the $0.2/k^{0.5}$ diminishing stepsize by a slight margin. 
We next zoom into the 25\%--75\% percentile performance and note that constant stepsizes are slightly better. While the $\ell_2$ errors are relatively similar in scale, constant stepsizes have better coverage probabilities which can be explained by the wider confidence intervals. This suggests that inference with constant stepsize iterates gives better covariance estimates, even under i.i.d.\ data. Nonetheless, the median confidence interval coverage probabilities for two constant stepsizes is around 95\%, which is around its targeted coverage probability.  

\begin{table}[htbp]
    \centering
    \begin{tabular}{ |c|c|c|c|c|c| } 
        \hline
        \multicolumn{2}{|c|}{\multirow{2}{*}{Percentile (\%)} }&\multicolumn{4}{|c|}{Comparison table} \\
        \cline{3-6}
        \multicolumn{2}{|c|}{}& 0.2 & 0.02  & $0.2/\sqrt{k}$ & $0.02/\sqrt{k}$\\ 
        \hline
        \multirow{3}{*}{10} & $\ell_2$ & 0.98 & 0.83   & 0.89 & 0.91\\ 
        & CI & 1.48 & 1.33  & 1.22 & 0.90\\ 
        & Cov & 91 & 91 & 86 & 58\\ 
        \hline
        \multirow{3}{*}{25} & $\ell_2$ & 1.16 & 1.07 & 1.06& 1.12\\ 
        & CI & 1.81 & 1.61  & 1.54 & 1.17\\ 
        & Cov & 92 & 92 &89 & 70\\ 
        \hline
        \multirow{3}{*}{50} & $\ell_2$ & 1.47 & 1.31  & 1.36 & 1.45\\ 
        & CI & 2.70 & 2.38  & 2.16 & 1.55\\ 
        & Cov & 94 & 94  & 92&77\\ 
        \hline
         \multirow{3}{*}{75} & $\ell_2$ & 2.21 &1.94 &1.84& 2.09\\ 
        & CI &4.086 & 3.33 & 2.99& 2.54\\ 
        & Cov  & 95 & 95  & 93& 82\\ 
        \hline
        \multirow{3}{*}{90} & $\ell_2$ & 5.28 &3.17  & 3.37 & 7.31\\ 
        & CI& 9.35 & 6.30  & 5.27 & 4.79\\ 
        & Cov & 90 & 96 & 95& 89\\ 
        \hline
    \end{tabular}
    \caption{Inference comparison of different stepsize regimes under i.i.d.\ Data.
    $\ell_2$ and ``CI" values are of unit $10^{-3}$. Both CI width and coverage probability listed in the table is for the 1-st coordinate estimate only.
    }
    \label{tab:quantile table iid}
\end{table}

\newpage
\subsection{Setup for Bootstrapping Experiments}
\label{sec:bootstrap-details}
Recall that our goal is to find $\theta^\ast$ that solves
\begin{equation*}
    \E_{x\sim\pi}[A(x)]\theta+\E_{x\sim\pi}[b(x)]=0,
\end{equation*}
where $(x_t)_{t\geq0}$ is a Markov chain and $\pi$ is its stationary distribution.

In this work, we assume that we only have access to Markovian underlying data $(x_t)_{t\geq0}$.
Therefore, we first simulate a trajectory of the Markov chain of length $10^6$, and we store the entire trajectory. We also shuffle the data to mimic an i.i.d.\ sampling paradigm. Next, we sample with replacement a batch size of $10^4$ from the stored states, i.e., $\{x_{(1)},\ldots, x_{(10^4)}\}$. Using these sampled states, we obtain estimates of $\BarA$ and $\Bar{b}$,
\begin{equation*}
    \hat{A}=\frac{1}{10^4}\sum_{k=1}^{10^4}A(x_{(k)}),\quad\text{and}\quad \hat{b}=\frac{1}{10^4}\sum_{k=1}^{10^4}b(x_{(k)}).
\end{equation*}
As such, for this batch, we obtain a point estimate
\begin{equation*}
    \hat\theta=\hat{A}^{-1}\hat{b}.
\end{equation*}

Next, we repeat the above sampling step 500 times and obtain $\{\hat\theta_{(i)}\}_{i=1}^{500}$. We can then estimate $\bar\theta$ and $\bar\Sigma_\theta$ from the 500 point estimates of $\theta^\ast$,
and subsequently construct the confidence interval as
\begin{equation*}
    \Bigg[\bar\theta_j-z_{1-\frac{q}{2}}\sqrt{\frac{\hat\Sigma_{j,j}}{K(n-n_0)}}, \,\bar\theta_j+z_{1-\frac{q}{2}}\sqrt{\frac{\hat\Sigma_{j,j}}{K(n-n_0)}}\Bigg],
\end{equation*}
where $\bar\theta_j$ denotes the $j$-th coordinate of $\bar\theta:=\frac{1}{500}\sum_{i=1}^{500}\hat\theta_{(i)}$, and $\hat\Sigma_{j,j}$ refers to the $(j,j)$-th entry of the estimated covariance matrix $\hat{\Sigma}:=\frac{1}{500}\sum_{i=1}^{500}(\hat\theta_{(i),j}-\bar\theta_j)^2$.

\newpage
\subsection{QQ Plots to Verify Normality}
In this last section, we conduct experiments to examine the normality of the averaged iterates of Markovian LSA under constant stepsize and provide quantile-quantile (QQ) plots to visually verify our CLT results. 

We examine a Markovian LSA with $|\cX|=5$ and $d=3$. We then run 1000 different realized trajectories $(x_k)_{k\geq0}$ of length $10^5$ and run LSA iterates with the above-described stepsize regimes with the initial burn-in period selected as $b=100$ and zero discard length, i.e., $n_0=0$, in between each batch, and a constant stepsize $\alpha=0.2$ for this experiment. We focus on the point estimate $\bar\theta$ obtained from the $1000$ independent runs, and generate quantile-quantile plots for each coordinate against a fitted normal distribution.
These plots are included below.

\begin{figure}[ht]

\centering
\includegraphics[width=.3\textwidth]{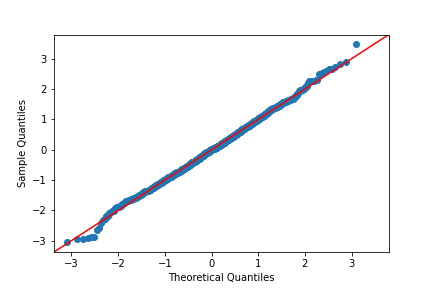}\hfill
\includegraphics[width=.3\textwidth]{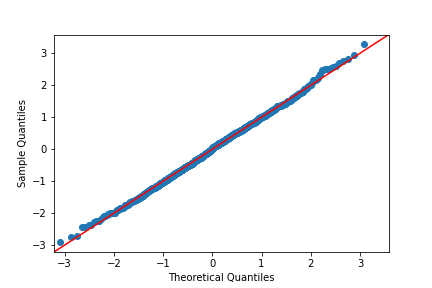}\hfill
\includegraphics[width=.3\textwidth]{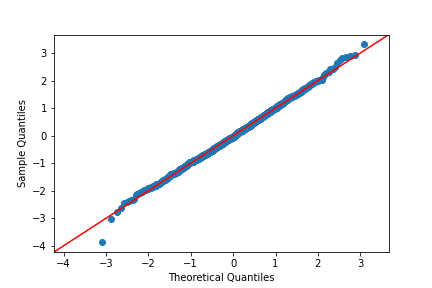}

\caption{QQ Plots for Three Coordinates.}
\label{fig:qqplots}

\end{figure}

These plots provide a clear visual confirmation that the averaged LSA iterates adhere to a normal distribution, which aligns with our CLT Theorem~\ref{thm:CLT}.

\end{document}